\documentclass[lettersize,journal]{IEEEtran}
\def\BibTeX{{\rm B\kern-.05em{\sc i\kern-.025em b}\kern-.08em
    T\kern-.1667em\lower.7ex\hbox{E}\kern-.125emX}}
\usepackage{balance}

\usepackage{csquotes}
\usepackage[english]{babel}
\usepackage{amsmath,amsfonts,amssymb}
\usepackage{empheq}
\usepackage{dsfont}
\usepackage{mathrsfs}
\usepackage{siunitx} 
\usepackage[ruled,linesnumbered,vlined]{algorithm2e}
\usepackage{multirow}
\usepackage{float}
\usepackage{array}
\usepackage{diagbox}
\usepackage{setspace}
\usepackage{catchfile}
\usepackage[normalem]{ulem}
\usepackage{subcaption}
\usepackage{subfiles} 
\usepackage{nameref}
\usepackage[bookmarks,hidelinks]{hyperref} 
\usepackage[capitalize]{cleveref}
\Crefname{line}{line}{lines}
\Crefname{assumption}{Assumption}{Assumptions}
\usepackage[c2]{optidef}
\usepackage{xcolor}

\definecolor{brightpink}{rgb}{1.0, 0.0, 0.5}

\newcommand{\revise}[1]{{{\color{black} #1}}}

\usepackage{tikz}
\usepackage{tikz-3dplot}
\usepackage{pgfplots}
\pgfplotsset{compat=newest}
\usepgfplotslibrary{statistics}
\usepgfplotslibrary{fillbetween}
\usepgfplotslibrary{groupplots}
\usetikzlibrary{shapes,calc,positioning,spy}
\def\addlegendimage{\csname pgfplots@addlegendimage\endcsname}

\usetikzlibrary{external}
\makeatletter
\AtEndDocument{%
  \protected@write\@auxout{}{\string\gdef\string\lastfig{\number\value{figure}}}%
}
\makeatother

\title{Maximum-Volume \\ Nonnegative Matrix Factorization} 
\date{}

\author{
  Olivier {Vu Thanh}\textsuperscript{1,2}\thanks{\revise{
    \textsuperscript{1}French geological survey (BRGM), Avenue Claude Guillemin 3, 45100 Orléans, France.} \textsuperscript{2}University of Mons, Rue de Houdain 9, 7000 Mons, Belgium. Emails: \href{mailto:o.vuthanh@brgm.fr}{o.vuthanh@brgm.fr}, \href{mailto:nicolas.gillis@umons.ac.be}{nicolas.gillis@umons.ac.be}. The authors acknowledge the support by the European Research Council (ERC consolidator, eLinoR, no 101085607). OVT was supported by an F.R.S.-FNRS FRIA grant when this paper was written. 
  }
  \and
  Nicolas Gillis\textsuperscript{2}
}

\newtheorem{theorem}{Theorem}
\newtheorem{corollary}{Corollary}
\newtheorem{definition}{Definition}

\newtheorem{remark}{Remark}

\newtheorem{proof}{Proof}

\newcommand{\overbar}[1]{\mkern 1.5mu\overline{\mkern-1.5mu#1\mkern-1.5mu}\mkern 1.5mu}

\DeclareMathOperator{\conv}{conv}

\DeclareMathOperator*{\argmin}{argmin}
\DeclareMathOperator{\cone}{cone}
\DeclareMathOperator{\rank}{rank}

\DeclareMathOperator{\col}{col}

\DeclareMathOperator{\logdet}{logdet}
\DeclareMathOperator{\tr}{tr}
\DeclareMathOperator{\diag}{Diag}

\DeclareMathOperator*{\dom}{dom}

\newcommand{\R}{\mathbb{R}}

\newcommand{\Wt}{{W^\top}}

\newcommand{\Wold}{W_o}
\newcommand{\Hold}{H_o}
\newcommand{\Wextra}{\overbar{W}}
\newcommand{\Hextra}{\overbar{H}}
\newcommand{\Wextraold}{\overbar{W}_o}
\newcommand{\Hextraold}{\overbar{H}_o}
\newcommand{\La}{\mathcal{L}}

\definecolor{DarkGreen}{rgb}{0.13,0.55,0.13}%

\begin{document}

\markboth{}%
{}

\IEEEpubid{}

\maketitle

\begin{abstract}
    Nonnegative matrix factorization (NMF) is a popular data embedding technique. Given a nonnegative data matrix $X$, it aims at finding two lower dimensional matrices, $W$ and $H$, such that $X\approx WH$, where the factors $W$ and $H$ are constrained to be element-wise nonnegative. The factor $W$ serves as a basis for the columns of $X$. In order to obtain more interpretable and unique solutions, minimum-volume NMF (MinVol NMF) minimizes the volume of $W$.     
    In this paper, we consider the dual approach, where   the volume of $H$ is maximized instead; this is referred to as maximum-volume NMF (MaxVol NMF).  
    MaxVol NMF is identifiable under the same conditions as MinVol NMF in the noiseless case, 
    but it behaves rather differently in the presence of noise. 
    In practice, MaxVol NMF is much more effective to extract a sparse decomposition and does not generate rank-deficient solutions. 
    In fact, we prove that the solutions of MaxVol NMF with the largest volume correspond to clustering the columns of $X$ in disjoint clusters, while the solutions of MinVol NMF with smallest volume are rank deficient. 
    We propose two algorithms to solve MaxVol NMF. We also present a normalized variant of MaxVol NMF that exhibits better performance than MinVol NMF and MaxVol NMF, and can be interpreted as a continuum between standard NMF and orthogonal NMF. 
    We illustrate our results in the context of hyperspectral unmixing. The code is available from {\color{blue}\url{https://gitlab.com/vuthanho/maxvolmf.jl}}. 
\end{abstract}


\section{Introduction}\label{sec:introduction}

Let $X\in\R^{m\times n}$ be a flattened hyperspectral data cube, where $m$ is the number of spectral bands and $n$ is the number of pixels. Hyperspectral unmixing (HU) aims at finding the set of spectral signatures of the materials present in the scene, called endmembers, and their corresponding proportion in each pixel, called abundances. These endmembers and abundances are respectively stored in the matrices $W\in\R^{m\times r}$ and $H\in\R^{r\times n}$, where $r$ is the number of endmembers. If we assume that interactions other than linear mixing are negligible, the data can be modeled as $X=WH+N$, where $N$ is a noise matrix. Retrieving $W$ and $H$ is a challenging task that requires other assumptions than low-rank; in particular, nonnegative matrix factorization (NMF)~\cite{lee1999learning} has shown to be effective for this task~\cite{bioucas2012hyperspectral}. The nonnegativity assumption is motivated by the fact that the spectral reflectance of materials is nonnegative, and that the abundances of materials can only be additive. Another assumption is the minimum-volume criterion, originally thought by~\cite{full1981extended} and the so-called Craig's belief~\cite{craig1994minimum}. The idea is that, in the absence of pure pixels (that is, pixels containing a single endmember), finding endmembers whose cone or convex hull tightly contains the data points retrieves the true endmembers. 
The minimum-volume criterion has also been used successfully in other applications, such as blind audio source separation \cite{leplat2019blind,wang2021minimum} and topic modeling~\cite{fu2016robust, fu2018anchor}. 
Regardless of the application, the minimum-volume criterion encourages interpretability of the features since they are close to the data points.

Maximum-volume NMF (MaxVol NMF), 
introduced in~\cite{tatli2021polytopic}, maximizes the volume of $H$ and can be interpreted as a dual approach of MinVol NMF~\cite{abdolali2024dual}; 
see Section~\ref{sec:motivation} for more details. 
However, to the best of our knowledge, its practical behavior, in particular on HU, has not been explored much. 
However, we will see that, in the presence of noise, MinVol NMF and MaxVol NMF \revise{behave} rather differently. In particular,  MaxVol NMF is much more effective to extract a sparse factor $H$ (that is, sparse abundances) and does not generate rank-deficient solutions.

\paragraph*{Outline and contribution of the paper} In \Cref{sec:motivation}, we motivate, introduce and analyze MaxVol NMF. In \Cref{sec:solvemaxvol}, we propose two algorithms to solve MaxVol NMF. In \Cref{sec:normmaxvolnmf}, we present a normalized variant of MaxVol NMF that exhibits better performance than MinVol NMF and MaxVol NMF, which we illustrate in the context of HU in Section~\ref{sec:normmaxvolnmfexp}. We conclude and discuss future works in \Cref{sec:conclusion}.

\section{Motivation: MinVol vs.\ MaxVol NMF}\label{sec:motivation}
\IEEEpubidadjcol

In this section, we first highlight two weaknesses of MinVol NMF, and then introduce MaxVol NMF and how it avoids these two pitfalls.

\subsection{Two weaknesses on MinVol NMF} 

In the absence of noise, MinVol NMF is defined as follows 
\begin{equation} \label{eq:exactminvol} 
    \min_{W,H} \det(W^\top W) \; \text{ such that } \; 
    X= WH, W \geq 0, \; H\in\Delta^{r\times n}, 
\end{equation}
where $\Delta^{r \times n}$ is the set of $r \times n$ matrices whose columns lies in the probability simplex,  $\Delta^r=\{x\in\R^r \ | \ x\geq0,1_r^\top x=1\}$, 
$1_r\in\R^r$ being the vector of all ones. 
This means that the simplex structure is imposed on the columns of $H$; this is called the sum-to-one constraint in the HU literature. 
The quantity $\sqrt{\det(W^\top W)}/r!$ is the volume of the columns of $W$ and the origin within the $r$-dimensional column space of $W$. 

\revise{
\begin{remark}[Beware of sum-to-one rows for $H$] 
In this paper, we do not consider the sum-to-one constraint on the rows of $H$ for MinVol NMF as studied in~\cite{fu2018identifiability}. We explain in details in Appendix~\ref{app:SSCrownorm} why. The reason is an overlooked/unknown fact in the literature: \emph{making the rows of $H$ sum to one often destroys the identifiability} for imbalanced data sets (that is, when endmembers are present in rather different proportions). 
\end{remark}
}

In practice, in the presence of noise, MinVol NMF needs to balance the volume and the data fitting terms, and the following optimization problem is often considered 
\begin{mini}|s|
	{\scriptstyle W,H}{\|X - WH\|_F^2 + \lambda \logdet(\Wt W + \delta I )}{\label{eq:approxminvol}}{}
	\addConstraint{W \geq 0}{,~H\in\Delta^{r\times n}, }{}
\end{mini}
where $\lambda$ is a penalty parameter, and $\delta$ is a small constant preventing $\logdet$ to be unbounded from below. The use of the logdet has algorithmic and practical advantages~\cite{fu2016robust, ang19}. 

Let us discuss two weaknesses of MinVol NMF. First, the MinVol criterion introduces a bias that can reduce the quality of the unmixing. 
\revise{In fact, regularization is well-known to induce bias, e.g., the $\ell_1$ norm penalty in sparse regression shrinks regression coefficients while inducing sparsity~\cite{fan2001variable}.  
This follows 
directly from the tradeoff between the data-fitting and volume terms in~(2) .  
The reduction of the MinVol penalty will typically come at the price of an increase in the data-fitting error. For example, in the exact case, where we are looking for a zero data fitting term (that is, $X = WH$), the optimal solution of~(2)  will degrade the data fitting term as $\lambda$ increases in order to promote the MinVol penalty.   
}

Let us illustrate this with the Samson dataset that contains mostly three endmembers:  water, soil and tree. 
The spectral signature of the water has a low magnitude relatively to the spectral signature of the soil and tree. Hence a bad estimation of the water spectral signature does not increase significantly the reconstruction error. Decreasing the norm of the spectral signature of the water is therefore an easy way to decrease the volume of $W$, and it can be done with only a slight increase in reconstruction error. 
This can be seen on \Cref{fig:lambda1minvol}, where the spectral signature of water (in red) for MinVol NMF with $\lambda=1$ contains $24$ zeros, while there should not be any zeros because there is not a wavelength at which water absorbs completely electromagnetic energy. Here, increasing $\lambda$ will only worsen this behavior, e.g., 36 zeros for $\lambda = 50$. 


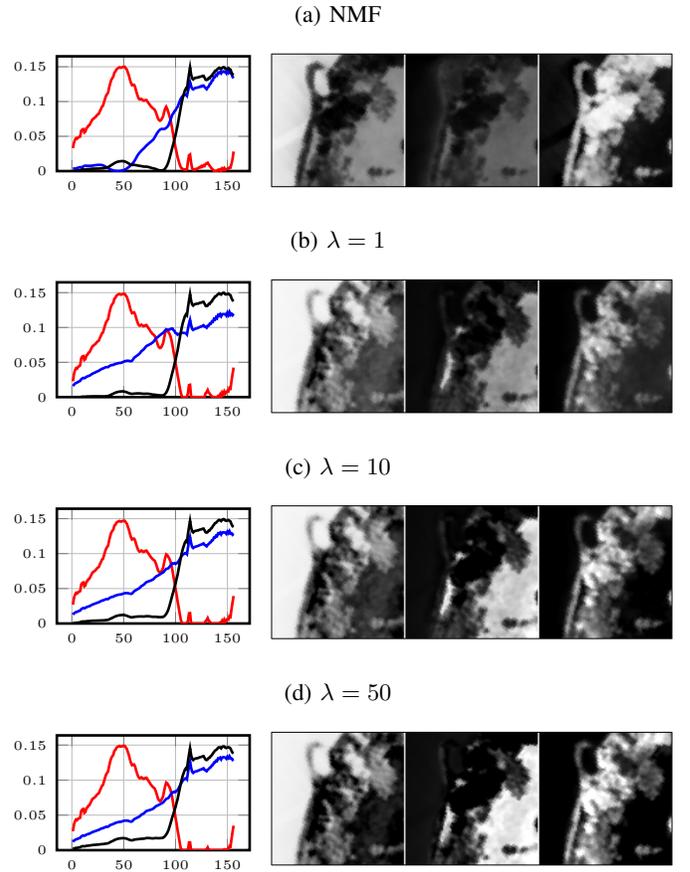
\begin{figure}
\begin{subfigure}{\linewidth}
    \caption{NMF}
    \begin{minipage}[t]{0.39\linewidth}
        \vspace{0pt}
        	\begin{tikzpicture}
        		\begin{axis}[cycle list name=color list,width=1.2\linewidth,height=0.9\linewidth,line width = 1pt,no markers,grid=major,
                    tick label style={/pgf/number format/fixed,font=\tiny},ymin=0]
        		\foreach \col in {0,...,2} {
        			\addplot table [x expr=\coordindex+1, y index=\col] {maxvolnmf/figures/samson_minvol_0.txt};
        		}
        		\end{axis}
        	\end{tikzpicture}
    \end{minipage}
    \begin{minipage}[t]{0.6\linewidth}
        \vspace{0pt}
            \fbox{\includegraphics[width=\linewidth]{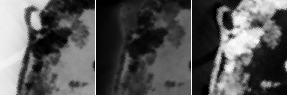}}
    \end{minipage}
\end{subfigure}

\vspace{0.2cm}

\begin{subfigure}{\linewidth}
    \caption{$\lambda=1$}
    \begin{minipage}[t]{0.39\linewidth}\vspace{0pt}
        	\begin{tikzpicture}
        		\begin{axis}[cycle list name=color list,width=1.2\linewidth,height=0.9\linewidth,line width = 1pt,no markers,grid=major,
                    tick label style={/pgf/number format/fixed,font=\tiny},ymin=0]
        		\foreach \col in {0,...,2} {
        			\addplot table [x expr=\coordindex+1, y index=\col] {maxvolnmf/figures/samson_minvol_1.txt};
        		}
        		\end{axis}
        	\end{tikzpicture}
    \end{minipage}
    \begin{minipage}[t]{0.6\linewidth}\vspace{0pt}
            \fbox{\includegraphics[width=\linewidth]{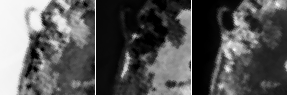}}
    \end{minipage}
    \label{fig:lambda1minvol}
\end{subfigure}

\vspace{0.2cm}

\begin{subfigure}{\linewidth}
    \caption{$\lambda=10$}
    \begin{minipage}[t]{0.39\linewidth}\vspace{0pt}
        \begin{tikzpicture}
            \begin{axis}[cycle list name=color list,width=1.2\linewidth,height=0.9\linewidth,line width = 1pt,no markers,grid=major,
                tick label style={/pgf/number format/fixed,font=\tiny},ymin=0]
            \foreach \col in {0,...,2} {
                \addplot table [x expr=\coordindex+1, y index=\col] {maxvolnmf/figures/samson_minvol_10.txt};
            }
            \end{axis}
        \end{tikzpicture}
    \end{minipage}
    \begin{minipage}[t]{0.6\linewidth}\vspace{0pt}
            \fbox{\includegraphics[width=\linewidth]{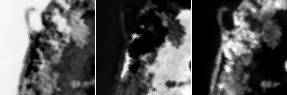}}
    \end{minipage}
\end{subfigure}

\vspace{0.2cm}

\begin{subfigure}{\linewidth}
    \caption{$\lambda=50$}
   \begin{minipage}[t]{0.39\linewidth}\vspace{0pt}
        \begin{tikzpicture}
            \begin{axis}[cycle list name=color list,width=1.2\linewidth,height=0.9\linewidth,line width = 1pt,no markers,grid=major,
                tick label style={/pgf/number format/fixed,font=\tiny},ymin=0]
            \foreach \col in {0,...,2} {
                \addplot table [x expr=\coordindex+1, y index=\col] {maxvolnmf/figures/samson_minvol_50.txt};
            }
            \end{axis}
        \end{tikzpicture}
    \end{minipage}
   \begin{minipage}[t]{0.6\linewidth}\vspace{0pt}
            \fbox{\includegraphics[width=\linewidth]{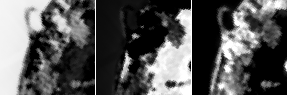}}
    \end{minipage}
\end{subfigure}
\caption[Abundance maps and endmembers for MinVol on Samson]{Abundance maps and normalized endmembers (from the left to the right: {\color{red}water}, {\color{blue}soil} and tree) for MinVol on the Samson dataset with $\delta=1$.}
\label{fig:samsonseverallambdaminvol}
\end{figure}
\IEEEpubidadjcol
\revise{Second, minimizing the volume of $W$ does not provide direct control over the sparsity of $H$. This can be understood geometrically. Recall that, when the columns of $W$ are affinely independent and $X(:,j)=Wh_j$ with $h_j\in\Delta^r$, the entries of $h_j$ are the barycentric coordinates of $X(:,j)$ in $\conv(W)$. Hence, a data point lying in the interior of $\conv(W)$ has strictly positive barycentric coordinates, and therefore $h_j$ is positive (dense).  In contrast, if the data point $X(:,j)$ lies on a facet of $\conv(W)$, at least one entry of $h_j$ is zero.  
Therefore, reducing the volume of $\conv(W)$ while keeping the data points inside it tends to bring the data closer to its facets and indirectly promote sparsity in $H$. However, in the noisy setting considered in~\eqref{eq:approxminvol}, the data points are not constrained to belong to
$\conv(W)$. As $\lambda$ increases, the volume term may shrink $\conv(W)$ to the point where some data points lie outside it. In that case, there is no reason for these abundances to become sparser as the volume decreases. Consequently, increasing $\lambda$ does not necessarily increase the sparsity of $H$.} 
This behavior is illustrated in~\Cref{fig:samsonseverallambdaminvol}.
As $\lambda$ increases, the corresponding abundance maps become slightly crisper, but the improvement in terms of sparsity remains limited and comes with a worse spectral signature for water. 

Now consider another, less noisy, dataset, the Moffett dataset. 
On \Cref{fig:moffettseverallambdaminvol}, we observe that the abundance map for NMF is not perfect, with mixed materials in the abundance maps. 
Adding the MinVol criterion with $\lambda=1$ improves the decomposition and, as a consequence, the sparsity. Still, the water and tree extraction are not correct, as there are some detected water within the lands where it should in fact be trees. Increasing $\lambda$ to 10 slightly improves this, but the water artifacts are still there. Increasing $\lambda$ further does not improve the unmixing. With $\lambda=50$, one of the columns of $W$  collapses to zero. 
If a practitioner has some a priori knowledge on the sparsity of the decomposition, MinVol NMF cannot explicitly control sparsity, though sparsity is often desired in unmixing. 

In this paper, we will see how MaxVol NMF preserves the spirit of MinVol NMF without the aforementioned weaknesses.

\begin{figure}
\begin{subfigure}{\linewidth}
    \caption{NMF}
   \begin{minipage}[t]{0.39\linewidth}\vspace{0pt}
        \begin{tikzpicture}
            \begin{axis}[ymax=1, xmin=0, xmax=156, cycle list name=color list,width=1.2\linewidth,height=0.9\linewidth,line width = 1pt,no markers,grid=major,
                tick label style={/pgf/number format/fixed,font=\tiny},ymin=0]
            \foreach \col in {0,...,2} {
                \addplot table [x expr=\coordindex+1, y index=\col] {maxvolnmf/figures/moffett_minvol_0.txt};
            }
            \end{axis}
        \end{tikzpicture}
    \end{minipage}
   \begin{minipage}[t]{0.6\linewidth}\vspace{0pt}
            \fbox{\includegraphics[width=\linewidth]{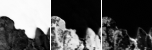}}
    \end{minipage}
\end{subfigure}

\vspace{0.2cm}

\begin{subfigure}{\linewidth}
    \caption{$\lambda=1$}
   \begin{minipage}[t]{0.39\linewidth}\vspace{0pt}
        \begin{tikzpicture}
            \begin{axis}[ymax=1, xmin=0, xmax=156, cycle list name=color list,width=1.2\linewidth,height=0.9\linewidth,line width = 1pt,no markers,grid=major,
                tick label style={/pgf/number format/fixed,font=\tiny},ymin=0]
            \foreach \col in {0,...,2} {
                \addplot table [x expr=\coordindex+1, y index=\col] {maxvolnmf/figures/moffett_minvol_1.txt};
            }
            \end{axis}
        \end{tikzpicture}
    \end{minipage}
   \begin{minipage}[t]{0.6\linewidth}\vspace{0pt}
            \fbox{\includegraphics[width=\linewidth]{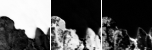}}
    \end{minipage}
\end{subfigure}

\vspace{0.2cm}

\begin{subfigure}{\linewidth}
    \caption{$\lambda=10$}
   \begin{minipage}[t]{0.39\linewidth}\vspace{0pt}
        \begin{tikzpicture}
            \begin{axis}[ymax=1, xmin=0, xmax=156, cycle list name=color list,width=1.2\linewidth,height=0.9\linewidth,line width = 1pt,no markers,grid=major,
                tick label style={/pgf/number format/fixed,font=\tiny},ymin=0]
            \foreach \col in {0,...,2} {
                \addplot table [x expr=\coordindex+1, y index=\col] {maxvolnmf/figures/moffett_minvol_10.txt};
            }
            \end{axis}
        \end{tikzpicture}
    \end{minipage}
   \begin{minipage}[t]{0.6\linewidth}\vspace{0pt}
            \fbox{\includegraphics[width=\linewidth]{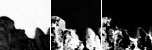}}
    \end{minipage}
\end{subfigure}

\vspace{0.2cm}

\begin{subfigure}{\linewidth}
    \caption{$\lambda=50$}
   \begin{minipage}[t]{0.39\linewidth}\vspace{0pt}
        \begin{tikzpicture}
            \begin{axis}[ymax=1, xmin=0, xmax=156, cycle list name=color list,width=1.2\linewidth,height=0.9\linewidth,line width = 1pt,no markers,grid=major,
                tick label style={/pgf/number format/fixed,font=\tiny},ymin=0]
            \foreach \col in {0,...,2} {
                \addplot table [x expr=\coordindex+1, y index=\col] {maxvolnmf/figures/moffett_minvol_50.txt};
            }
            \end{axis}
        \end{tikzpicture}
    \end{minipage}
   \begin{minipage}[t]{0.6\linewidth}\vspace{0pt}
            \fbox{\includegraphics[width=\linewidth]{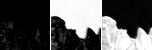}}
    \end{minipage}
\end{subfigure}
\caption[Abundance maps and endmembers for MinVol on Moffett]{Abundance maps and normalized endmembers (from the left to the right: {\color{red}water}, {\color{blue}soil} and tree, except for $\lambda=50$) for MinVol on the Moffett dataset with $\delta=0.1$.}
\label{fig:moffettseverallambdaminvol}
\end{figure}
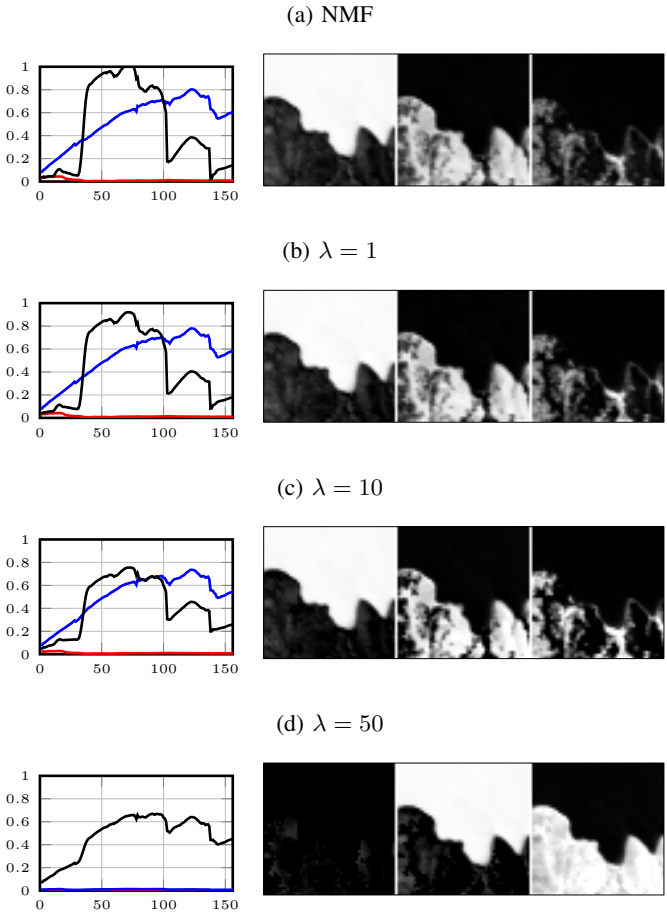

\subsection{MaxVol NMF}\label{sec:maxvolnmf}

Let us introduce MaxVol NMF through its equivalence with MinVol NMF in the exact case. Consider the NMF $X=\overbar{W}\overbar{H}$ where $\overbar{W}$ has full column rank. 
For any full column rank matrix $W$ with the same column space as $\overbar{W}$, there exists an invertible matrix $Q$ such that $W=\overbar{W}Q$. Then, $$\det(W^\top W)=\det(Q^\top \overbar{W}^\top \overbar{W}Q)=\det(Q)^2\det(\overbar{W}^\top \overbar{W}).$$ Minimizing $\det(W^\top W)$ is equivalent to minimizing $\det(Q)^2\det(\overbar{W}^\top \overbar{W})$. 
Hence, given $X=\overbar{W}\overbar{H}$, computing the exact MinVol NMF of $X$ is equivalent to solving
\begin{equation} \label{eq:exactQminvol} 
    \min_Q \det(Q)^2 
    \quad \text{ such that } 
    \quad 
\overbar{W}Q\geq0, Q^{-1}\overbar{H}\in\Delta^{r\times n}. 
\end{equation}
Minimizing the quantity $\det(Q)^2$ is equivalent to maximizing the quantity \revise{$\det(Q^{-1})^2 = \det(Q)^{-2}$}. 
To sum up, in the exact case, minimizing the volume of $W$ is equivalent to maximizing the volume of $H$. This leads us to the MaxVol NMF formulation in the absence of noise:  
\begin{equation} \label{eq:exactmaxvol} 
\max_{W,H} \det(HH^\top) 
\; \text{ such that } \; 
X=WH, \; W \geq 0, \; H\in\Delta^{r\times n}. 
\end{equation}
We now discuss the identifiability of MaxVol NMF, and its practical behavior. 

\subsection{Identifiability of MaxVol NMF}\label{sec:identifiability}

In the absence of noise, MaxVol NMF is just as identifiable as MinVol NMF~\cite{lin15, fu15}, under the same condition, defined as follows. 
\begin{definition}[Sufficiently scattered condition - SSC]  \label{app:def:ssc}  
The matrix $H \in \mathbb{R}^{r \times n}_+$ is sufficiently scattered if the following two conditions are satisfied:\index{sufficiently scattered condition!definition} 

[SSC1]  $\mathcal{C} = \{x \in \mathbb{R}^r_+ \ | \ e^\top x \geq \sqrt{r-1} \|x\|_2 \} \; \subseteq \; \cone(H)$.

[SSC2]  There does not exist any orthogonal matrix $Q$ such that $\cone(H) \subseteq \cone(Q)$, except for permutation matrices. 
\end{definition} 

\revise{
The SSC requires that the columns of $H$ are sufficiently spread within the nonnegative orthant $\mathbb R^r_+$ so as to contain the cone $\mathcal C$ that is tangent to every facet of $\mathbb R^r_+$; see, e.g., the discussion in~\cite{fu2019nonnegative}, and  also~Appendix~\ref{app:SSCrownorm} for visual examples.  
}

\begin{theorem}\cite{tatli2021polytopic}
	\label{th:uniquemaxvol}
	Let $X=WH$ be a MaxVol NMF of $X$ of size $r = \rank(X)$, in the sense of~\eqref{eq:exactmaxvol}. If $H$ satisfies the SSC, then MaxVol NMF $(W,H)$ of $X$ is essentially unique.
\end{theorem}
\begin{proof} 
\revise{This is a special case of the more general theorem for maximum-volume polytopic matrix factorization from~\cite{tatli2021polytopic}. This is also a corollary of MinVol NMF identifiability and the equivalence between minimizing $\det(W^\top W)$ and maximizing $\det(HH^\top)$ when $X=WH$; see above.} 
\end{proof}

\subsection{Behavior of MaxVol NMF}\label{sec:behaviormaxvol}

In the inexact case, we consider the following MaxVol NMF formulation: 
\begin{mini}|s|
{\scriptstyle W,H}
{f(W,H):=\frac{1}{2}\| X - WH \|_F^2 -\lambda\logdet(HH^\top+\delta I)\label{eq:nonexactmaxvolmf}}{}{}
\addConstraint{W\geq0,H\in\Delta^{r \times n}.}
\end{mini}  
It should be noted that, unlike MinVol NMF, from an optimization perspective, the \revise{parameter} $\delta$ in the $\logdet$ is not needed anymore. Maximizing the $\logdet$ \revise{penalizes rank deficiency in $H$}. \revise{In particular, when $\delta = 0$, we can guarantee that any optimal $H$ will be full rank since any rank-deficient $H$ would have an infinite objective function value}.   
Still, we keep $\delta$ in our model because it has some physical meaning; see \Cref{sec:normmaxvolnmf}.

To understand the main difference between MinVol NMF and MaxVol NMF, consider the asymptotic case when $\lambda$ goes to infinity. For MinVol NMF, $W$ will converge to the all-zero matrix~\cite{leplat19} for any $\delta > 0$, in order to minimize 
\[
\logdet(W^\top W + \delta I)  
= \sum_{i=1}^r \log (\sigma_i^2(W) + \delta), 
\]
where $\sigma_i(W)$ is the $i$th singular value of $W$.  For MaxVol NMF, as $\lambda$ goes to infinity, we can show that $H$ converges to a matrix whose rows are mutually orthogonal, and such that the $l_2$ norm of each row are as close to each other as possible. 
This is shown in Appendix~\ref{app:optsolA}. 
Due to this result, and assuming $n$ is a \revise{multiple} of $r$ (to simplify the presentation), that is, $n=dr$ for $d\in\mathbb{N}$, 
increasing $\lambda$ in \eqref{eq:nonexactmaxvolmf} will make $HH^\top$ converge to a diagonal matrix whose elements are all equal to $d$. 
In other words, the rows of $H$ will be mutually orthogonal while the simplex constraint on the columns of $H$ imposes that $H(i,j)\in\{0,1\}$. The squared $\ell_2$  norm of each row is equal to the number of non-zero elements in the corresponding row, $d$. 
From the HU point of view, one pixel will be assigned to only one material, this is a hard clustering where every cluster \emph{must have the same size}. 
The clustering behavior of MaxVol NMF is interesting and offers more control over the sparsity of the decomposition than MinVol NMF. Also, maximizing the volume of $H$ indirectly minimizes the volume of $W$ without the drawback of potentially setting a useful endmember to zero due to its low reflectance or linear dependence with other endmembers. 
However, the fact that increasing $\lambda$ leads to a clustering with clusters of the same size is a clear weakness. For this reason, we will introduce a normalized MaxVol NMF variant, N-MaxVol NMF, in \Cref{sec:normmaxvolnmf} that allows uneven clusters when the penalty parameter $\lambda$ goes to infinity. 

\Cref{fig:samsonmaxvollambdas} shows an experiment for the Samson data set. Increasing $\lambda$ intensifies the clustering, until a hard clustering is achieved with $\lambda=50$. Increasing $\lambda$ removes some of the false positives for water, but not all of them. 
The pixels that are wrongly assigned to water are actually pixels that are in the shadow, which makes their reflectance low. The simplex structure of $H$ prevents the model from assigning pixels to their correct endmember but to one with lower magnitude (in this case, water).  
Also, the improvement of the abundance map of the water is at the cost of a hard clustering, while a soft clustering would be preferable to properly unmix soil and tree. 
In \Cref{sec:normmaxvolnmf}, we will present an improved variant of MaxVol NMF to mitigate this issue.  
\begin{figure}[htb]
\centering
    \begin{subfigure}{\linewidth}
        \centering
        \caption{$\lambda=0.5$}
        \fbox{\includegraphics[width=\linewidth]{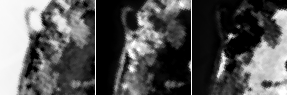}}
    \end{subfigure}

    \vspace{0.2cm}

    \begin{subfigure}{\linewidth}
        \centering
        \caption{$\lambda=5$}
        \fbox{\includegraphics[width=\linewidth]{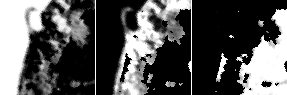}}
    \end{subfigure}

    \vspace{0.2cm}

    \begin{subfigure}{\linewidth}
        \centering
        \caption{$\lambda=10$}
        \fbox{\includegraphics[width=\linewidth]{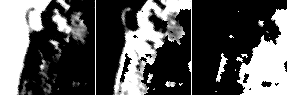}}
    \end{subfigure}

    \vspace{0.2cm}

    \begin{subfigure}{\linewidth}
        \centering
        \caption{$\lambda=50$}
        \fbox{\includegraphics[width=\linewidth]{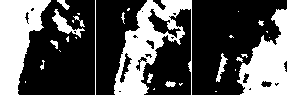}}
    \end{subfigure}
\caption{Abundance maps of MaxVol NMF on Samson, depending on $\lambda$.}
\label{fig:samsonmaxvollambdas}
\end{figure}


\begin{remark}
    \label{remark:samsonmaxvollambdas}
    About the results in \Cref{fig:samsonmaxvollambdas}:
    \begin{itemize}
        \item $\lambda$ is tuned using \cite{nguyen2024towards}, where convergence is assumed when the relative difference of the objective function between two successive iterates is below $10^{-4}$, and where the maximum variation of $\lambda$ is capped to $10\%$. 
        \item In \Cref{sec:solvemaxvol} we show two different algorithms to solve MaxVol NMF. The abundance maps displayed on \Cref{fig:samsonmaxvollambdas} are the same regardless of the used algorithm, except for $\lambda=50$ where the adaptive gradient method crashes, probably due to some numerical issues. The ADMM based algorithm still works well with $\lambda=50$.
    \end{itemize}
\end{remark}

\section{Solving MaxVol NMF} \label{sec:solvemaxvol}

The most common strategy to solve matrix factorizations problems such as MaxVol NMF ~\eqref{eq:nonexactmaxvolmf} is to use block coordinate descent schemes. 
Here, we consider two blocks: $W$ and $H$. Compared to MinVol NMF~\eqref{eq:approxminvol}, the main difficulty in solving~\eqref{eq:nonexactmaxvolmf} holds in the term $-\lambda\logdet(\cdot)$. 

For MinVol NMF, since $X\rightarrow \logdet(X)$ is concave, it is possible to derive a quadratic majorizer  relatively to $W$ whose gradient is Lipschitz continuous, and then use gradient descent on this majorizer~\cite{fu2016robust, leplat19}. In fact, the first-order Taylor approximation $\logdet(X)$ at the current iterate provides such a majorizer. 

Unfortunately, $-\logdet(.)$ is not concave, which prevents from using the aforementioned updating strategy for~\eqref{eq:nonexactmaxvolmf}. 
In this section, we propose two algorithms to solve~\eqref{eq:nonexactmaxvolmf}. The first algorithm in~\cref{sec:adgrad} is adapted from~\cite{malitsky2020adaptive}. Its core idea is to approximate the local Lipschitzness by using the previous iterate and to compute the corresponding Lipschitz gradient descent. The second algorithm is based on the Alternating Direction Method of Multipliers (ADMM).
The reason we introduce two algorithms is as follows: 
although ADMM will work better on average, we will be able to adapt the first algorithm to our new proposed N-MaxVol NMF model from \Cref{sec:normmaxvolnmf}. 


\subsection{Adaptive accelerated gradient descent}
\label{sec:adgrad}

Our first proposed algorithm for~\eqref{eq:nonexactmaxvolmf} relies on~\cite[Alg.~2]{malitsky2020adaptive}. This algorithm uses the previous iterate to approximate the local Lipschitzness and derive an appropriate step size. The previous iterate is also used to induce some extrapolation. The only knowledge that is needed from the objective function is its gradient. It should be noted that~\cite[Alg. 2]{malitsky2020adaptive} is only designed to update all variables at once. 
In our case, it would mean that $[\Wt,H]\in\R^{r\times(m+n)}$ should be updated all at once. Most gradient-based algorithms for constrained matrix factorization are using a two block alternating strategy, performing several updates on $W$ and then several updates on $H$. By doing so, computational savings are possible by precomputing some quantities  that remain unchanged during the update of one block (e.g., $XH^\top$ and $HH^\top$ when updating $W$). 
We will follow this common two-block strategy and, all we need are the gradients of MaxVol NMF~\eqref{eq:nonexactmaxvolmf}: 
\begin{align*}
    \nabla f(W) & = \frac{\partial f}{\partial W}(W) = (WH-X)H^\top, \\
    \nabla f(H) &= \frac{\partial f}{\partial H}(H) = W^\top(WH\text{$-$}X)\text{$-$}2\lambda(HH^\top\text{$+$}\delta I)^{-1}H. 
\end{align*} 
Our adaptation of~\cite[Alg. 2]{malitsky2020adaptive} with a two-block strategy is given in~\cref{alg:Adgrad2}\revise{, where $\Hold$ denotes the previous iterate, $H$ the current iterate, and $\overbar{H}$ the extrapolation of the current iterate.} 
\begin{algorithm}[htbp!]
\caption{Adgrad2}
\label{alg:Adgrad2}
\DontPrintSemicolon
\KwIn{data matrix $X \in \mathbb{R}^{m \times n}$, initial factors $\Wold \in \mathbb{R}_+^{m \times r}$  and $\Hold \in \Delta^{r \times n}$}
$\Gamma_{\Wold}=\|\Hold \Hold^\top\|,\gamma_{\Wold}={\Gamma_{\Wold}}^{-1},\theta_W=\Theta_W=10^9,\Wextraold=\Wold,\overbar{W}=W=[\Wold-10^{-6}\nabla f(\Wold)]_+$\;
$\Gamma_{Ho}=\|\Wold^\top \Wold\|,\gamma_{Ho}={\Gamma_{Ho}}^{-1},\theta_H=\Theta_H=10^9,\Hextraold=\Hold,\overbar{H}=H=[\Hold-10^{-6}\nabla f(\Hold)]_{\Delta^{r \times n}}$\;
\For{$k=1,2,\dots$\nllabel{alg:Adgrad2:line:outerloop}}{
    \While{stopping criteria not satisfied\nllabel{alg:Adgrad2:line:Winnerloop}}{
        $\gamma_{W}=\min\left(\gamma_{\Wold}\sqrt{1+\frac{\theta_W}{2}},\frac{\|\Wextra-\Wextraold\|_F}{2\|\nabla f(\Wextra)-\nabla f(\Wextraold)\|_F}\right)$\;
        $\Gamma_{W}=\min\left(\Gamma_{\Wold}\sqrt{1+\frac{\Theta_W}{2}},\frac{\|\nabla f(\Wextra)-\nabla f(\Wextraold)\|_F}{2\|\Wextra-\Wextraold\|_F}\right)$\;
        $W=[\Wextra-\gamma_{W}\nabla f(\Wextra)]_+$\;
        $\theta_W=\gamma_{W}/\gamma_{\Wold},\Theta_W=\Gamma_{W}/\Gamma_{\Wold}$\;
        $\Wextraold=\Wextra$\;
        $\Wextra=W+\frac{1-\sqrt{\gamma_{W}\Gamma_{W}}}{1+\sqrt{\gamma_{W}\Gamma_{W}}}(W-\Wold)$\;
        $\Wold=W$\;
        $\gamma_{\Wold}=\gamma_{W},\Gamma_{\Wold}=\Gamma_{W}$\;
    }
    \While{stopping criteria not satisfied\nllabel{alg:Adgrad2:line:Hinnerloop}}{
        $\gamma_{H}=\min\left(\gamma_{\Hold}\sqrt{1+\frac{\theta_H}{2}},\frac{\|\Hextra-\Hextraold\|_F}{2\|\nabla f(\Hextra)-\nabla f(\Hextraold)\|_F}\right)$\;
        $\Gamma_{H}=\min\left(\Gamma_{\Hold}\sqrt{1+\frac{\Theta_H}{2}},\frac{\|\nabla f(\Hextra)-\nabla f(\Hextraold)\|_F}{2\|\Hextra-\Hextraold\|_F}\right)$\;
        $H=[\Hextra-\gamma_{H}\nabla f(\Hextra)]_{\Delta^{r \times n}}$\label{alg:lineupdateH}\;
        $\theta_H=\gamma_{H}/\gamma_{\Hold},\Theta_H=\Gamma_{H}/\Gamma_{\Hold}$\;
        $\Hextraold=\Hextra$\;
        $\Hextra=H+\frac{1-\sqrt{\gamma_{H}\Gamma_{H}}}{1+\sqrt{\gamma_{H}\Gamma_{H}}}(H-\Hold)$\;
        $\Hold=H$\;
        $\gamma_{\Hold}=\gamma_{H},\Gamma_{\Hold}=\Gamma_{H}$\;
    }
}
\end{algorithm}

\begin{remark}
    The adaptive part is mostly useful for the update of $H$. In order to update $W$, which requires to solve a convex nonnegative least squares problem, any other algorithm could be used instead of the \textbf{while} 
    loop in~\cref{alg:Adgrad2:line:Winnerloop}.
\end{remark}

\subsection{Alternating direction method of multipliers (ADMM)}
\label{sec:admm}

Let us introduce the variable $Y = HH^\top \in \mathbb{R}^{r \times r}$ and let $\Lambda \in \mathbb{R}^{r \times r}$ be the Lagrange multipliers corresponding to this constraint, and consider the following ADMM reformulation of~\eqref{eq:nonexactmaxvolmf}:
\begin{equation}\label{eq:admmmaxvol}
\begin{array}{cl}
    \underset{\scriptstyle W,H,Y,\Lambda}{\min}  & \La(W,H,Y,\Lambda):=\frac{1}{2}\| X - WH \|_F^2 \\
    & \quad-\lambda\logdet(Y+\delta I)+\langle Y-HH^\top,\Lambda\rangle \\
    & \quad +\frac{\rho}{2}\|Y-HH^\top\|_F^2 \\
    &\\
    \text{s.t.}  & W\geq0,H\in\Delta^{r \times n}.
\end{array}
\end{equation}
ADMM consists of the following updates~\cite{bertsekas2016nonlinear}: 
\begin{align}
    W^{k+1} &= \argmin_{W\geq0}\La(W,H^k,Y^k,\Lambda^k) \label{eq:admm_W}\\
    H^{k+1} &= \argmin_{H\revise{\in} \Delta^{r \times n}}\La(W^{k+1},H,Y^k,\Lambda^k) \label{eq:admm_H}\\
    Y^{k+1} &= \argmin_{Y}\La(W^{k+1},H^{k+1},Y,\Lambda^k) \label{eq:admm_Y}\\
    \Lambda^{k+1} &= \Lambda^{k} + \rho(Y^{k+1}-H^{k+1}{H^{k+1}}^\top) \label{eq:admm_Lambda}
\end{align}

\paragraph{Updating $W$}

Like in~\cref{sec:adgrad}, the update for $W$ can be computed through any algorithm for constrained convex problems, as~\eqref{eq:admm_W} is equivalent to $$W^{k+1} = \argmin_{W\geq0} \frac{1}{2}\|X-WH^{k}\|^2_F.$$ 
Here we propose to use TITAN~\cite{hien2023inertial}, an extrapolated first-order algorithm with convergence guarantees, like in \cite{vuthanh2021inertial}. The resulting update for $W^{k+1}$ is detailed in~\cref{alg:admm_titan_W} 
\begin{algorithm}[htb!]
\caption{Update of $W$ with TITAN}
\label{alg:admm_titan_W}
\DontPrintSemicolon
\KwIn{$\alpha_1,X,H^k,W,\Wold$}
\KwOut{$W$}
$L_W=\|H^k{H^k}^\top\|$\;
\While{stopping criteria not satisfied}{
    $\alpha_0=\alpha_1$\;
    $\alpha_1=\frac{1}{2}(1+\sqrt{1+4\alpha_0^2})$\;
    $\beta=\frac{\alpha_0-1}{\alpha_1}$\;
    $\Wextra=W+\beta(W-\Wold)$\;
    $\Wold=W$\;
    $W=[\Wextra+\frac{1}{L_W}(XH^\top-\Wextra H^k{H^k}^\top)]_+$
}
\end{algorithm}

\paragraph{Updating $H$}

We propose two ways of updating $H$. The first one consists of solving directly~\eqref{eq:admm_H} with the adaptive accelerated gradient descent algorithm described in~\cref{sec:adgrad}. 
The second one consists of deriving a non-Euclidean gradient method. We derive a Bregman surrogate of $H\rightarrow\La(W^{k+1},H,Y^k,\Lambda^k):=\La(H)$ and update $H$ by minimizing this surrogate. 
The main motivation to use such a surrogate is that there does not exist a Lipschitz surrogate of $\La$ relatively to $H$. The gradient of $H\rightarrow\La(W^{k+1},H,Y^k,\Lambda^k)$ is clearly not Lipschitz continuous because the gradient of \revise{$\|Y-HH^\top\|_F^2$} relatively to $H$ is cubic. Although $H\rightarrow\La(H)$ is not $L$-smooth, using the framework of~\cite{bauschke2017descent}, we can show that it is smooth relatively to the quartic norm kernel proposed in~\cite{dragomir2021quartic}.
\begin{definition}[Bregman distance] Given a convex function  $h$, dubbed a distance kernel, the corresponding Bregman distance is defined as 
    $$D_h(x,y)=h(x)-h(y)-\langle\nabla h(y),x-y\rangle.$$
\end{definition}
Note that $D_h$ is not a proper distance as it is asymmetric. 
\begin{definition}[Relative smoothness~\cite{bauschke2017descent}]
    We say that a differentiable function $f:\R^{r\times n}\rightarrow\R$ is $L$-smooth relatively to the distance kernel $h$ if there exists $L>0$ such that for every $X,Y\in\R^{r\times n}$,
    $$f(X)\leq f(Y) + \langle\nabla f(Y),X-Y\rangle+L D_h(X,Y).$$
    If $f$ is twice differentiable, $L$-smoothness relatively to $h$ is equivalent to 
    $$\nabla^2f(X)[U,U]\leq L\nabla^2 h(X)[U,U]\quad\forall X,U\in\R^{r\times n},$$
    where $\nabla^2f(X)[U,U]$ denotes the second derivative of $f$ at $X$ in the direction $U$.
\end{definition}
First, we focus on the relative smoothness of the quartic term. According to~\cite{dragomir2021quartic} we have
\begin{multline}
	\label{eq:bregsurroquarticpart}
	\frac{1}{2}\|Y-HH^\top\|_F^2:=g(H)\leq g(H^k)+\langle\nabla g(H^k),H-H^k\rangle \\+ D_h(H,H^k), 
\end{multline}
where $\nabla g(H^k)=2(H^k{H^k}^\top-Y)H^k$, 
and $h(H)=\frac{\alpha}{4}\|H\|_F^4+\frac{\sigma}{2}\|H\|_F^2$ with $\alpha=6$ and $\sigma=2\|Y\|_2$. Substituting~\eqref{eq:bregsurroquarticpart} in~\eqref{eq:admmmaxvol}, 
\begin{multline}
	\label{eq:lagrangianquarticsurrogate}
	\La(H)\leq u_{H^k}(H):=\frac{1}{2}\|X-WH\|_F^2-\langle HH^\top,\Lambda\rangle\\ +\rho\langle\nabla g(H^k),H\rangle+\rho h(H)-\rho\langle\nabla h(H^k),H \rangle+C_H
\end{multline}
where $C_H$ is a constant relatively to $H$. Let us compute the second directional derivative of $u_{H^k}$,
\begin{align}
    \nabla^2 u_{H^k}(H)[U,U]
        & =  \langle(W^\top W-2\Lambda^\top)U,U\rangle+\rho\sigma\|U\|_F^2\nonumber\\
            & \qquad +\rho\alpha(\|H\|_F^2\|U\|_F^2+2\langle H,U\rangle^2) \nonumber\\
        & \leq \rho\alpha(\|H\|_F^2\|U\|_F^2+2\langle H,U\rangle^2) \nonumber\\
            & \qquad +(\|W^\top W-2\Lambda^\top\|_2+\rho\sigma)\|U\|_F^2 \nonumber\\
        & = \nabla^2\left(\frac{\tilde{\alpha}}{4}\|H\|_F^4+\frac{\tilde{\sigma}}{2}\|H\|_F^2\right)[U,U],\label{eq:kernel1smooth}
\end{align}
where $\tilde{\alpha}=\rho\alpha$ and $\tilde{\sigma}=\rho\sigma+\|W^\top W-2\Lambda^\top\|_2$.
From~\eqref{eq:kernel1smooth} and~\eqref{eq:lagrangianquarticsurrogate}, $H\rightarrow\La(H)$ is 1-smooth relatively to the kernel $\tilde{h}:H\rightarrow\frac{\tilde{\alpha}}{4}\|H\|_F^4+\frac{\tilde{\sigma}}{2}\|H\|_F^2$. More explicitly, 
\begin{equation*}
	\La(H)\leq u_{H^k}(H^k)+\langle\nabla u_{H^k}(H^k),H-H^k\rangle+D_{\tilde{h}}(H,H^k).
\end{equation*}

The update for $H$ is then obtained by minimizing the aforementioned surrogate 
\begin{align}
	H^{k+1}&=\argmin_{H\in\Delta^{r \times n}}\left\{ \langle\nabla u_{H^k}(H^k),H\rangle\text{$+$}\tilde{h}(H)\text{$-$}\langle\nabla \tilde{h}(H^k),H\rangle \right\}, \nonumber\\
	\label{eq:Hminbregsurro}
	&=\argmin_{H\in\Delta^{r \times n}}\left\{t_k(H):=\tilde{h}(H)-\langle Q^k,H\rangle \right\},
\end{align}
where $Q^k=\nabla \tilde{h}(H^k)-\nabla u_{H^k}(H^k)$. This is equivalent to the Bregman proximal iteration map described in~\cite{dragomir2021quartic} with a step size equal to 1.

\begin{corollary}
	\label{cor:bregupdateHform}
	The solution of~\eqref{eq:Hminbregsurro} is of the form $$H^{k+1}=\frac{1}{\tilde{\alpha}\|H^{k+1}\|_F^2+\tilde{\sigma}}[Q^k-1_r\nu^\top]_+,$$
	where $\nu\in\R^{n}$.
\end{corollary}
\begin{proof}
Consider the Lagrangian of~\eqref{eq:Hminbregsurro}
$$\La_{t_k}(H,\Lambda,\nu)=t_k(H)-\langle H,\Lambda\rangle+\langle H^\top 1_r-1_n,\nu\rangle$$
where $\Lambda\in\R^{r\times n}_+$ and $\nu\in\R^n$. According to the KKT optimality conditions:

\begin{empheq}[left=\empheqlbrace]{align}
	H^{k+1} &\in \Delta^{r \times n}, \\
	\langle\Lambda^*,H^{k+1}\rangle &= 0, \label{eq:kkt12}\\
	\nabla t_k(H^{k+1})-\Lambda^*+1_r{\nu^*}^\top &= 0, \label{eq:kkt13}  
\end{empheq}

\begin{empheq}[left=\Leftrightarrow\empheqlbrace]{align}
	H^{k+1} &\in \Delta^{r \times n}, \\
	\langle\nabla\tilde{h}(H^{k+1})-Q^k+1_r{\nu^*}^\top,H^{k+1}\rangle &= 0, \label{eq:kkt22}\\
	\nabla\tilde{h}(H^{k+1})-Q^k+1_r{\nu^*}^\top &\geq 0, \label{eq:kkt23}
\end{empheq}
where~\eqref{eq:kkt22} is coming from substituting~\eqref{eq:kkt13} in~\eqref{eq:kkt12}, and~\eqref{eq:kkt23} is coming from the fact that $\Lambda^*\geq0$. First, combining~\eqref{eq:kkt22} and~\eqref{eq:kkt23}, we have
\begin{equation}
	\label{eq:kkt_hadamard}
	(\nabla\tilde{h}(H^{k+1})-Q^k+1_r{\nu^*}^\top) \circ H^{k+1} = 0, 
\end{equation}
where $\circ$ is the Hadamard product. For all $p$ in $1,\dots,r$, for all $j$ in $1,\dots,n$,
\begin{enumerate}
	\item if $Q^k(p,j)-\nu^*_j<0$, $\nabla\tilde{h}(H^{k+1})(p,j)-(Q^k(p,j)-\nu^*_j)>0$ because $\nabla\tilde{h}(H)=(\tilde{\alpha}\|H\|_F^2+\tilde{\sigma})H\geq0$, then \eqref{eq:kkt_hadamard} $\Rightarrow H^{k+1}(p,j)=0$,
	
	\item if $Q^k(p,j)-\nu^*_j=0$, $\nabla\tilde{h}(H^{k+1})(p,j)=(\tilde{\alpha}\|H^{k+1}\|_F^2+\tilde{\sigma})H^{k+1}(p,j)$ so~\eqref{eq:kkt_hadamard} $\Rightarrow H^{k+1}(p,j)=0$,
	
	\item if $Q^k(p,j)-\nu^*_j>0$, $\nabla\tilde{h}(H^{k+1})(p,j)=(\tilde{\alpha}\|H^{k+1}\|_F^2+\tilde{\sigma})H^{k+1}>0$ by~\eqref{eq:kkt23}, then~\eqref{eq:kkt_hadamard} $\Rightarrow \nabla\tilde{h}(H^{k+1})(p,j)-(Q^k(p,j)-\nu^*_j)=0\Leftrightarrow H^{k+1}(p,j)=\frac{Q^k(p,j)-\nu^*_j}{\tilde{\alpha}\|H^{k+1}\|_F^2+\tilde{\sigma}}$.
\end{enumerate}

In the end, $H^{k+1}=\frac{1}{\tilde{\alpha}\|H^{k+1}\|_F^2+\tilde{\sigma}}[Q^k-1_r{\nu^*}^\top]_+$.

\hfill $\square$ 
\end{proof}
In particular, $\nu$ in~\Cref{cor:bregupdateHform} is such that $1_r^\top[Q^k-1_r{\nu}^\top]_+=(\tilde{\alpha}\|H^{k+1}\|_F^2+\tilde{\sigma})1_n^\top\in\R^n$ since $1_r^\top H^{k+1}=1_n^\top$. In other words, $[Q^k-1_r{\nu}^\top]_+$ projects $Q$ on a scaled probability simplex where the scaling is equal to $\tilde{\alpha}\|H^{k+1}\|_F^2+\tilde{\sigma}$. How do we find $\nu$ since it depends on $H^{k+1}$? We propose to solve this inexactly with a simple fixed point algorithm where $\|H^{k+1}\|_F^2$ is the variable to optimize. The idea is that when $\|H^{k+1}\|_F^2$ is fixed, $\nu$ has a closed form solution. So for a fixed $\|H^{k+1}\|_F^2$ we compute $\nu$, then we update $\|H^{k+1}\|_F^2$ according to the new $\nu$ and repeat this process. The algorithm is described in~\Cref{alg:Hminbregsurro}. When $\|H^{k+1}\|_F^2$ is fixed, there are several algorithms that can compute exactly $\nu$. In~\cite{held1974validation}, the proposed algorithm requires to sort the entries of each column of $Q^k$. The main computational cost of this algorithm is this sorting. Once the sorting is completed, $\nu$ is found just by computing $n$ times the $\max$ between $r$ entries, which is linear. There exist faster algorithms like~\cite{condat2016fast} that do not rely on sorting. However, note that $Q^k$ is not changing in~\Cref{alg:Hminbregsurro}. Hence, using~\cite{held1974validation} to compute $\nu$ in \cref{alg:Hminbregsurro:nuline} only has a linear complexity if $Q^k$ is sorted only once before the \textbf{while} loop. In our code, $\epsilon$ is fixed to $10^{-6}$ and the \textbf{while} loop cannot exceed 100 iterations.

\begin{algorithm}[htb!]
	\caption{Algorithm to update $H$ in~\eqref{eq:Hminbregsurro}}
	\label{alg:Hminbregsurro}
	\SetKwInOut{Init}{init}
	\DontPrintSemicolon
	\KwIn{$Q^k,\tilde{\alpha},\tilde{\sigma}$}
	\Init{$\|H^{k+1}\|_F^2,\nu$}
	\KwOut{$H^{k+1}$}
	\While{$\frac{\|H^{k+1}\|_F^2-\left\|\frac{1}{\tilde{\alpha}\|H^{k+1}\|_F^2+\tilde{\sigma}}[Q^k-1_r\nu^\top]_+\right\|_F^2}{\|H^{k+1}\|_F^2}>\epsilon$}{
		compute $\nu$ such that $\frac{1}{\tilde{\alpha}\|H^{k+1}\|_F^2+\tilde{\sigma}}[Q^k-1_r\nu^\top]_+\in\Delta^{r\times n}$\label{alg:Hminbregsurro:nuline}\;
		update $\|H^{k+1}\|_F^2$ to $\left\|\frac{1}{\tilde{\alpha}\|H^{k+1}\|_F^2+\tilde{\sigma}}[Q^k-1_r\nu^\top]_+\right\|_F^2$\;
	}
	$H^{k+1}=\frac{1}{\tilde{\alpha}\|H^{k+1}\|_F^2+\tilde{\sigma}}[Q^k-1_r\nu^\top]_+$\;
\end{algorithm}

\paragraph{Updating $Y$}

The ADMM update of $Y^{k+1}$ is 
\begin{equation*}
Y^{k+1}=\argmin_{Y\succ-\delta I}-\lambda\logdet(Y+\delta I)+\langle Y,\Lambda\rangle+\frac{\rho}{2}\|Y-HH^\top\|_F^2.
\end{equation*}
Consider the change of variable $Z=Y+\delta I$,
\begin{multline}
Y^{k+1}+\delta I=\argmin_{Z\succ0}-\lambda\logdet(Z)\\+\frac{\rho}{2}\left\|Z-\left(HH^\top+\delta I -\frac{1}{\rho}\Lambda\right)\right\|_F^2. \label{eq:admmYafterchangevar}
\end{multline}
According to~\cite[Lemma 2.1]{wang2010solving},~\eqref{eq:admmYafterchangevar} has a closed form solution which is $$\Phi_{\frac{\lambda}{\rho}}^+\left(HH^\top+\delta I -\frac{1}{\rho}\Lambda\right)$$ where $\Phi_\gamma^+(x)=\frac{1}{2}(\sqrt{x^2+4\gamma}+x)$ and for a symmetric $A$ with an \revise{eigenvalue} decomposition $A=PDP^\top$ and $D=\diag(d)$, $\Phi_\gamma^+(A)=P\diag(\Phi_\gamma^+(d))P^\top$ where $\Phi_\gamma^+(d)$ is applied element-wise. In the end,
$$Y^{k+1}=\Phi_{\frac{\lambda}{\rho}}^+\left(HH^\top+\delta I -\frac{1}{\rho}\Lambda\right)-\delta I.$$


\subsection{Comparison of the two algorithms}

We now compare the two proposed algorithms for MaxVol NMF, both on synthetic and real datasets. The results are averaged over 10 runs and are presented on \Cref{fig:algos}. For the synthetic dataset, $W$ is drawn following a uniform distribution in $[0,1]$ and $H$ is such that each of its column is drawn following a Dirichlet distribution where the concentration parameters are all equal to $0.2$. The input matrix is $X=WH$. A different $X$ is drawn at each run. The compared algorithms are Adgrad2 (\Cref{sec:adgrad}), ADMM (\Cref{sec:admm}) and ADMM+Adgrad. ADMM+Adgrad has the same formulation as in~\eqref{eq:admmmaxvol}, but the update for $H$~\eqref{eq:admm_H} is performed using the adaptive gradient descent method instead of minimizing the proposed Bregman surrogate. Regardless of the dataset and of the algorithm, the number of iterations is fixed to 500, the number of inner iterations (that is, the number of times $H$ is updated before updating $W$, and vice-versa) is fixed to 20, $\lambda$ and $\delta$ are fixed to $1$. 
In \Cref{fig:algos}, on both synthetic data and Moffett, ADMM with $\rho=0.01$ has the best convergence speed and the lowest error. Still on synthetic data and Moffett, the proposed Bregman surrogate provides a nice approximation of the original ADMM formulation~\eqref{eq:admmmaxvol}. For equal $\rho$'s, ADMM always converges faster and to a lower error than ADMM+Adgrad. This experimentally justifies our choice for the use of a Bregman surrogate to update $H$ in the ADMM formulation of MaxVol NMF. However, this is at the cost of a higher computation time, due to \Cref{alg:Hminbregsurro}, as it can be seen in the reported average times in \Cref{tab:averagetimesynth}. 
One can always increase the tolerance threshold $\epsilon$ in \Cref{alg:Hminbregsurro}, but should remain careful. Let us increase $\epsilon$ to $10^{-3}$. The computation time of ADMM with $\rho=0.01$ is greatly reduced, as a run on the synthetic dataset only lasts 2.44s on average. However, for $\rho=0.1$ the algorithm diverges, as it can be seen on \Cref{fig:admm_bigger_epsi}, and the computation time is increased to 6.90s on average. Finally, ADMM is not always better than Adgrad2, like with Samson on \Cref{fig:algos_samson}. Reasons as to why one algorithm would be better than the other, and designing effective algorithms for MaxVol NMF are questions of further research. 


\begin{table}[htbp]
    \centering
    \begin{tabular}{l|c}
        Algorithm & Times (s)\\ \hline
        Adgrad2 & 3.67 \\ \hline
        ADMM+Adgrad $\rho=0.01$ & 2.88\\ \hline
        ADMM+Adgrad $\rho=0.1$ & 2.29\\ \hline
        ADMM $\rho=0.01$ & 5.33\\ \hline
        ADMM $\rho=0.1$ & 23.5
    \end{tabular}
    \caption{Average time per run on synthetic datasets.}
    \label{tab:averagetimesynth}
\end{table}

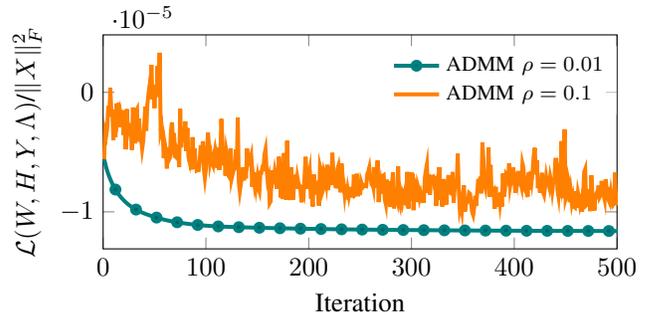
\begin{figure}[htbp!]
    \centering
    \begin{tikzpicture}
        \begin{axis}[
                    width=0.95\linewidth,
                    height=0.5\linewidth,
                    xmin = 0,
                    xmax = 500,
                    ylabel = {$\La(W,H,Y,\Lambda)$/$\|X\|_F^2$},
                    ytick = {0,-1e-5},
                    xlabel = {Iteration},
                    cycle list name=exotic,
                    mark size = 1.5pt,
                    mark repeat = 20,
                    legend cell align={left},
                    legend style={font=\footnotesize,at={(1,0.95)},anchor=north east,draw=none,fill opacity=0.5,text opacity=1}]
        \addplot+[mark phase=12,line width = 1.5pt] table[x expr=\coordindex+1, y index = 3]{maxvolnmf/figures/div_conv_mean_50_500_5.txt};
        \addplot+[mark = none,line width = 1.5pt] table[x expr=\coordindex+1, y index = 4]{maxvolnmf/figures/div_conv_mean_50_500_5.txt};
        \legend{ADMM $\rho=0.01$,ADMM $\rho=0.1$}
        \end{axis}
    \end{tikzpicture}
    \caption{ADMM on synthetic dataset with $\epsilon=10^{-3}$}
    \label{fig:admm_bigger_epsi}
\end{figure}

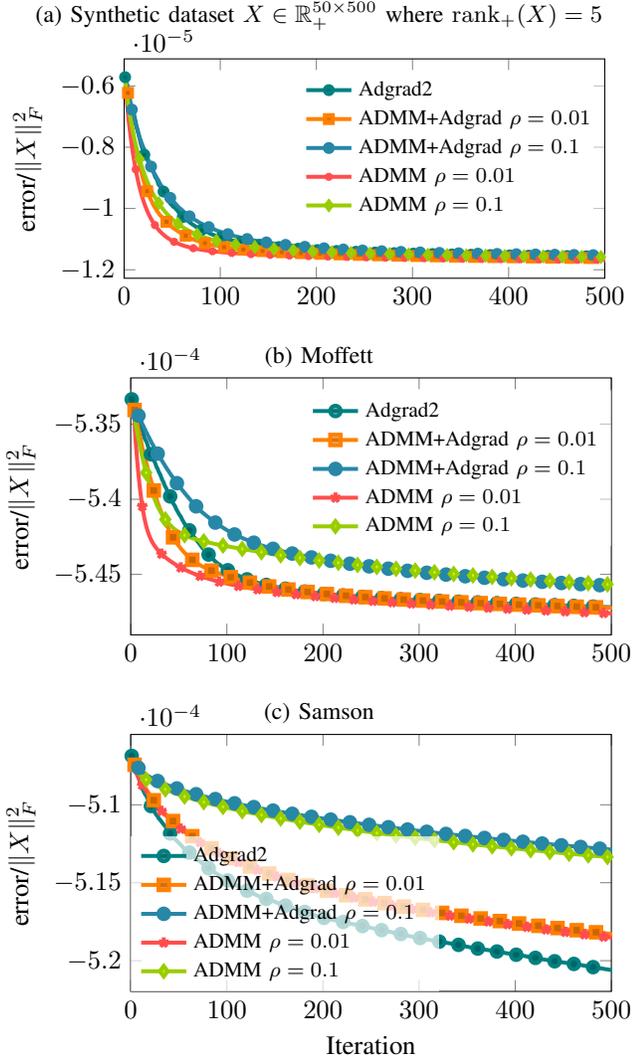
\begin{figure}[htbp!]
    \centering
    \begin{subfigure}{\linewidth}
        \centering
        \caption{Synthetic dataset $X\in\R^{50\times 500}_+$ where $\rank_+(X)=5$}
        \label{fig:algos_synth}
        \begin{tikzpicture}
            \begin{axis}[
                        width=0.9\linewidth,
                        height=4.5cm,
                        xmin = 0,
                        xmax = 500,
                        ylabel = error/$\|X\|_F^2$,
                        cycle list name=exotic,
                        mark size = 1.5pt,
                        mark repeat = 20,
                        legend cell align={left},
                        legend style={font=\footnotesize,at={(1,0.95)},anchor=north east,draw=none,fill opacity=0.5,text opacity=1}]
            \addplot+[mark phase= 0,line width = 1.5pt] table[x expr=\coordindex+1, y index = 0]{maxvolnmf/figures/conv_mean_50_500_5.txt};
            \addplot+[mark phase= 4,line width = 1.5pt] table[x expr=\coordindex+1, y index = 1]{maxvolnmf/figures/conv_mean_50_500_5.txt};
            \addplot+[mark phase= 8,line width = 1.5pt] table[x expr=\coordindex+1, y index = 2]{maxvolnmf/figures/conv_mean_50_500_5.txt};
            \addplot+[mark phase=12,line width = 1.5pt] table[x expr=\coordindex+1, y index = 3]{maxvolnmf/figures/conv_mean_50_500_5.txt};
            \addplot+[mark phase=16,line width = 1.5pt] table[x expr=\coordindex+1, y index = 4]{maxvolnmf/figures/conv_mean_50_500_5.txt};
            \legend{Adgrad2,ADMM+Adgrad $\rho=0.01$,ADMM+Adgrad $\rho=0.1$,ADMM $\rho=0.01$,ADMM $\rho=0.1$}
            \end{axis}
        \end{tikzpicture}
    \end{subfigure}

    \vspace{0.2cm}

    \begin{subfigure}{\linewidth}
        \centering
        \caption{Moffett}\vspace{-0.3cm}
        \label{fig:algos_moffett}
        \begin{tikzpicture}
            \begin{axis}[
                        width=0.9\linewidth,
                        height=5cm,
                        xmin = 0,
                        xmax = 500,
                        ylabel = error/$\|X\|_F^2$,
                        cycle list name=exotic,
                        mark size = 2pt,
                        mark repeat = 20,
                        legend cell align={left},
                        legend style={font=\footnotesize,at={(1,0.95)},anchor=north east,draw=none,fill opacity=0.5,text opacity=1}]
            \addplot+[mark phase= 0,line width = 1.5pt] table[x expr=\coordindex+1, y index = 0]{maxvolnmf/figures/conv_moffett.txt};
            \addplot+[mark phase= 4,line width = 1.5pt] table[x expr=\coordindex+1, y index = 1]{maxvolnmf/figures/conv_moffett.txt};
            \addplot+[mark phase= 8,line width = 1.5pt] table[x expr=\coordindex+1, y index = 2]{maxvolnmf/figures/conv_moffett.txt};
            \addplot+[mark phase=12,line width = 1.5pt] table[x expr=\coordindex+1, y index = 3]{maxvolnmf/figures/conv_moffett.txt};
            \addplot+[mark phase=16,line width = 1.5pt] table[x expr=\coordindex+1, y index = 4]{maxvolnmf/figures/conv_moffett.txt};
            \legend{Adgrad2,ADMM+Adgrad $\rho=0.01$,ADMM+Adgrad $\rho=0.1$,ADMM $\rho=0.01$,ADMM $\rho=0.1$}
            \end{axis}
        \end{tikzpicture}
    \end{subfigure}

    \vspace{0.2cm}

    \begin{subfigure}{\linewidth}
        \centering
        \caption{Samson}\vspace{-0.3cm}
        \label{fig:algos_samson}
        \begin{tikzpicture}
            \begin{axis}[
                        width=0.9\linewidth,
                        height=5cm,
                        xmin = 0,
                        xmax = 500,
                        ylabel = error/$\|X\|_F^2$,
                        xlabel = {Iteration},
                        cycle list name=exotic,
                        mark size = 2pt,
                        mark repeat = 20,
                        legend cell align={left},
                        legend style={font=\footnotesize,at={(0,0)},anchor=south west,draw=none,fill opacity=0.7,text opacity=1}]
            \addplot+[mark phase= 0,line width = 1.5pt] table[x expr=\coordindex+1, y index = 0]{maxvolnmf/figures/conv_samson.txt};
            \addplot+[mark phase= 4,line width = 1.5pt] table[x expr=\coordindex+1, y index = 1]{maxvolnmf/figures/conv_samson.txt};
            \addplot+[mark phase= 8,line width = 1.5pt] table[x expr=\coordindex+1, y index = 2]{maxvolnmf/figures/conv_samson.txt};
            \addplot+[mark phase=12,line width = 1.5pt] table[x expr=\coordindex+1, y index = 3]{maxvolnmf/figures/conv_samson.txt};
            \addplot+[mark phase=16,line width = 1.5pt] table[x expr=\coordindex+1, y index = 4]{maxvolnmf/figures/conv_samson.txt};
            \legend{Adgrad2,ADMM+Adgrad $\rho=0.01$,ADMM+Adgrad $\rho=0.1$,ADMM $\rho=0.01$,ADMM $\rho=0.1$}
            \end{axis}
        \end{tikzpicture}
    \end{subfigure}
    \caption{Comparison of algorithms for MaxVol NMF on various datasets}
    \label{fig:algos}
\end{figure}

\section{Normalized MaxVol NMF (N-MaxVol NMF)}\label{sec:normmaxvolnmf}

In \Cref{sec:maxvolnmf}, we explained that a main drawback of MaxVol \eqref{eq:nonexactmaxvolmf} is its bias towards clusters of the same size when $\lambda$ increases.  
Here, we introduce a normalized variant of MaxVol NMF to alleviate this issue, by maximizing the volume of the row-wise normalized $H$, denoted $\widetilde{H}$:  
\begin{mini}|s|
    {\scriptstyle W,H,\widetilde{H}}
    {f(W,H):=\frac{1}{2}\| X - WH \|_F^2 -\lambda\logdet(\widetilde{H}\widetilde{H}^\top+\delta I)\label{eq:nonexactnormalizedmaxvolmf}}{}{}
    \addConstraint{W\geq0,H\geq0,\revise{\|H(k,:)\|_2>0 \text{ for all }k}, \widetilde{H}=S^{-1}H}
    \addConstraint{\text{where } S=\diag(\|H(1,:)\|_2,\dots,\|H(r,:)\|_2).}
\end{mini} 
Let us explain why this model will not be biased towards  even clusterings. 
As $\lambda$ increases, one can show that $\widetilde{H}\widetilde{H}^\top$ converges to the identity; see Appendix~\ref{app:optsolB}.  
In other words, increasing $\lambda$ acts in favor of mutually orthogonal rows of $H$. Unlike MaxVol NMF, the norm of the rows of $H$ can take any non-zero value since it is $\widetilde{H}\widetilde{H}^\top$ that converges to the identity and not $HH^\top$. In fact, N-MaxVol NMF can be viewed as a continuum between NMF and Orthogonal NMF (ONMF), where ONMF requires $HH^\top = I_r$ and is equivalent to a weighted variant of spherical k-means~\cite{pompili2014two}. 

It is possible to control the range of the volume criterion via $\delta$. For $\delta>0$, 
$$
\logdet(\widetilde{H}\widetilde{H}^\top+\delta I)\in\left[\log((r+\delta)\delta^{r-1}),r\log(1+\delta)\right], 
$$ 
where the minimum and maximum are \revise{respectively} reached when $\widetilde{H}\widetilde{H}^\top=1_r1_r^\top$ and $\widetilde{H}\widetilde{H}^\top=I$; see Appendix~\ref{app:optsolB}. 
The parameter $\delta$ hence controls how large the volume of $\widetilde{H}$ can be, while $\lambda$ balances the reconstruction error and the volume criterion. 
By increasing $\delta$, the range is reduced, as it can be seen on \Cref{fig:range_normvol}. With respect to $\lambda$ and the reconstruction error, it is then harder to give more importance to the volume term in the objective. In the context of HU, $\delta$ can be seen as a mixture tolerance parameter, while $\lambda$ is a noise level estimation parameter. 
\begin{figure}[htbp!]
    \centering
    \begin{tikzpicture}
        \begin{axis}[
                    width=\linewidth,
                    height=6cm,
                    ymax = 1e2,
                    grid=major,
                    ymode=log,
                    xmin = 0,
                    xmax = 1,
                    ylabel = {},
                    xlabel = {$\delta$},
                    cycle list name=exotic,
                    legend columns = 2,
                    legend cell align={left},
                    legend style={font=\footnotesize,at={(1,0.95)},anchor=north east,draw=none,fill opacity=0.5,text opacity=1}]
        \foreach \j in {1,...,8}{
            \pgfmathtruncatemacro{\myresult}{\j+2}%
            \addplot+[mark=none,line width = 1.5pt] table[x index=0, y index = \j]{maxvolnmf/figures/range_normvol.txt};
            \addlegendentryexpanded{$r=\myresult$};
        }
        \end{axis}
    \end{tikzpicture}
    \caption[Range of the logdet depending on $r$ and $\delta$]{Value of $r\log(1+\delta)-\log(1+r\delta^{-1})-r\log\delta$ depending on $\delta$ for various $r$'s.}
    \label{fig:range_normvol}
\end{figure}
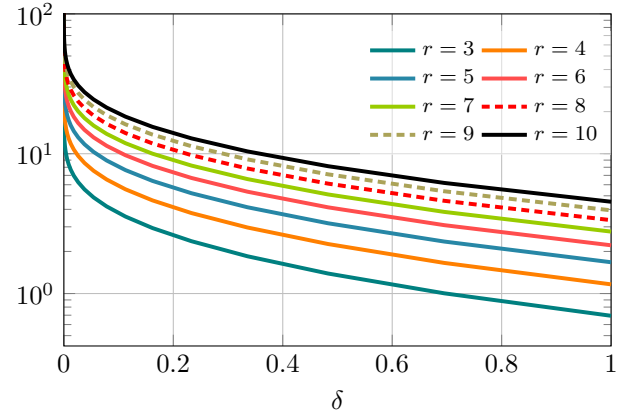

Another feature of N-MaxVol NMF is the removal of the simplex structure on $H$. 
Recall that the simplex structure is not without loss of generality. In HU for instance, if there are two pure pixels of the same material but one of them receives more light than the other, then a perfect unmixing would require a different endmember for each. In other words, the simplex structure might require a larger rank. 

In spite of the benefits the normalized variant brings, we ``lose'' two aspects of the vanilla MaxVol NMF. The most notable one is the identifiability. It remains unknown if N-MaxVol NMF is identifiable or not. We also lose the possibility to solve N-MaxVol NMF with the same ADMM formulation that we used for MaxVol NMF. We could not find a kernel that would provide us with a Bregman surrogate. Even if we did, the considered Bregman surrogate would then need to be nice enough to be easily solvable.  
We will therefore solve it with the adaptive accelerated gradient descent method.

\subsection{Solving N-MaxVol NMF}\label{sec:solvenormmaxvolnmf}

We solve N-MaxVol NMF \eqref{eq:nonexactnormalizedmaxvolmf} with the adaptive accelerated gradient descent method, introduced in \Cref{sec:adgrad}, 
and for which we only need to know the gradient which we provide in this section. 
The algorithm is exactly the same as \Cref{alg:Adgrad2}, except for the projected gradient step in \cref{alg:lineupdateH} that should be replaced by $H=[\Hextra-\gamma_{H}\nabla f(\Hextra)]_+$, where $f$ is the objective function, because there is no simplex structure in the normalized variant. It remains to compute the gradient of $f$ in \eqref{eq:nonexactnormalizedmaxvolmf} relatively to $H$. 
Knowing that 
\begin{align*}
    \frac{\partial \widetilde{H}(k,:)}{\partial H(k,j)}&=\begin{pmatrix}
        -\frac{H(k,1)H(k,j)}{\|H(k,:)\|^3} \\ \vdots \\ \frac{\|H(k,:)\|^2-H(k,j)^2}{\|H(k,:)\|^3} \\ \vdots \\ -\frac{H(k,n)H(k,j)}{\|H(k,:)\|^3}\end{pmatrix}^\top\\
    &=\frac{1}{\|H(k,:)\|^3}\left(\|H(k,:)\|^2 e_j^\top - H(k,j)H(k,:)\right),
\end{align*}
and using the chain rule, we have that

\begin{align*}
    \MoveEqLeft[1.5]\frac{\partial \logdet(\widetilde{H}\widetilde{H}^\top{+}\delta I)}{\partial H(k,j)} = \left\langle\frac{\partial \logdet(\widetilde{H}\widetilde{H}^\top{+}\delta I)}{\partial \widetilde{H}},\frac{\partial \widetilde{H}}{\partial H(k,j)}\right\rangle \\
    = & \left\langle 2(\widetilde{H}\widetilde{H}^\top{+}\delta I)^{-1}\widetilde{H},\frac{E_{k,j}}{\|H(k,:)\|} {-} \frac{H(k,j)}{\|H(k,:)\|^3}e_k H(k,:) \right\rangle \\
    = &\begin{multlined}[t]
        \frac{1}{\|H(k,:)\|}\langle 2(\widetilde{H}\widetilde{H}^\top+\delta I)^{-1}\widetilde{H},E_{k,j} \rangle \\ - \frac{1}{\|H(k,:)\|}\langle 2(\widetilde{H}\widetilde{H}^\top+\delta I)^{-1}, e_k \widetilde{H}(k,:)\widetilde{H}^\top \rangle \widetilde{H}(k,j).
    \end{multlined}
\end{align*}
In the end, 
\begin{equation*}
    \frac{\partial \logdet(\widetilde{H}\widetilde{H}^\top {+}\delta I)}{\partial H} \\= 2S^{-1}\left[(\widetilde{H}\widetilde{H}^\top {+}\delta I)^{-1}{-} \widetilde{D}_\delta\right]\widetilde{H}
\end{equation*}
and
\begin{equation*}
    \frac{\partial f}{\partial H}=W^\top(WH{-}X) {-} 2\lambda S^{-1}\left[(\widetilde{H}\widetilde{H}^\top {+}\delta I)^{-1}{-}\widetilde{D}_\delta\right]\widetilde{H},
\end{equation*}
where $\widetilde{D}_\delta = \diag\left((\widetilde{H}\widetilde{H}^\top +\delta I)^{-1}\widetilde{H}\widetilde{H}^\top\right)$.

\section{Numerical Experiments}\label{sec:normmaxvolnmfexp}


\pgfplotsset{
  boxplot legend/.style args={#1,#2}{ 
    legend image code/.code={
      \draw[#1,line width=0.7pt] 
        (0cm,-0.1cm) rectangle (0.6cm,0.1cm)
        (0.6cm,0cm) -- (0.7cm,0cm)
        (0cm,0cm) -- (-0.1cm,0cm)
        (0.7cm,0.1cm) -- (0.7cm,-0.1cm)
        (-0.1cm,0.1cm) -- (-0.1cm,-0.1cm)
        (0.3cm,-0.1cm) -- (0.3cm,0.1cm);
      \path[#1,mark options={#1,rotate=-90},mark=#2]
        plot coordinates {(0.9cm,0cm)};
    },
  },
}

In this section, we compare MinVol NMF, MaxVol NMF and N-MaxVol NMF on synthetic and hyperspectral datasets. The hyperparameter $\lambda$ is again tuned using  \cite{nguyen2024towards}, as specified in \Cref{remark:samsonmaxvollambdas}. All experiments can be run using the code available from {\color{blue}\url{https://gitlab.com/vuthanho/maxvolmf.jl}}. 

\subsection{Synthetic datasets}

The synthetic datasets are generated by mixing five endmembers from the USGS library. The ground truth $W$ contains the spectral signature of these endmembers, displayed on \Cref{fig:synth_endmembers}. We consider the two cases where $W$ is well conditioned, leaving the endmembers as they are in \Cref{fig:synth_endmembers}, and where $W$ is ill conditioned, scaling its columns linearly from $0.05$ to $1$. The columns of the ground truth abundance matrix $H$ are drawn following a Dirichlet distribution with parameter $\alpha=0.1$, and clipping the values below $0.05$ to zero. We consider the two cases where $H$ is column-wise stochastic and where it is not, scaling its columns randomly, where the scaling follows a normal distribution $\mathcal{N}(1,0.2^2)$.
This simulates the presence of different lighting conditions. The data is generated as $X=WH+N$ where $N$ is Gaussian noise. We consider three different signal-to-noise ratios (SNR): 30dB, 20dB and 10dB. For each of the four configurations, we generate 20 noiseless different datasets $WH$ and add the noise $N$ afterward. The estimated $W$ is compared to the ground truth $W$ by computing the maximum angle between the matched ground truth and estimated endmembers. In other words, among the endmembers, how accurate is the least accurately estimated endmember. The results are displayed as box plots on \Cref{fig:synth_max_angles}. As expected, N-MaxVol NMF outperforms MinVol NMF when $H$ is not stochastic, while MaxVol performs best when $H$ is stochastic followed by N-MaxVol NMF. 
This confirms our theoretical findings that MaxVol NMF is more suitable than MinVol NMF as it is not biased towards rank-deficient solutions while it allows to recover sparse $H$. 


\begin{figure}[htbp!]
    \centering
    \begin{tikzpicture}
        \begin{axis}[
                    width=\linewidth,
                    height=6cm,
                    ymin = 0, ymax = 1,
                    grid=major,
                    xmin = 1, xmax = 224,
                    ylabel = {},
                    cycle list name=exotic,
                    legend columns = 2,
                    legend cell align={left},
                    legend style={font=\footnotesize,at={(1,0.95)},anchor=north east,draw=none,fill opacity=0.5,text opacity=1}]
        \foreach \j in {0,...,4}{
            \addplot+[mark=none,line width = 1.5pt] table [x expr=\coordindex+1, y index=\j] {maxvolnmf/figures/DC2gt.txt};
        }
        \end{axis}
    \end{tikzpicture}
    \caption{Spectral signature of five endmembers from  the USGS library used to generate the synthetic datasets.} 
    \label{fig:synth_endmembers}
\end{figure}
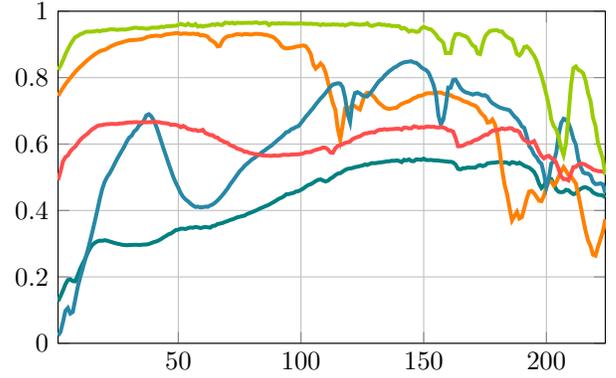

\newcount\nboxplots
\nboxplots=3 

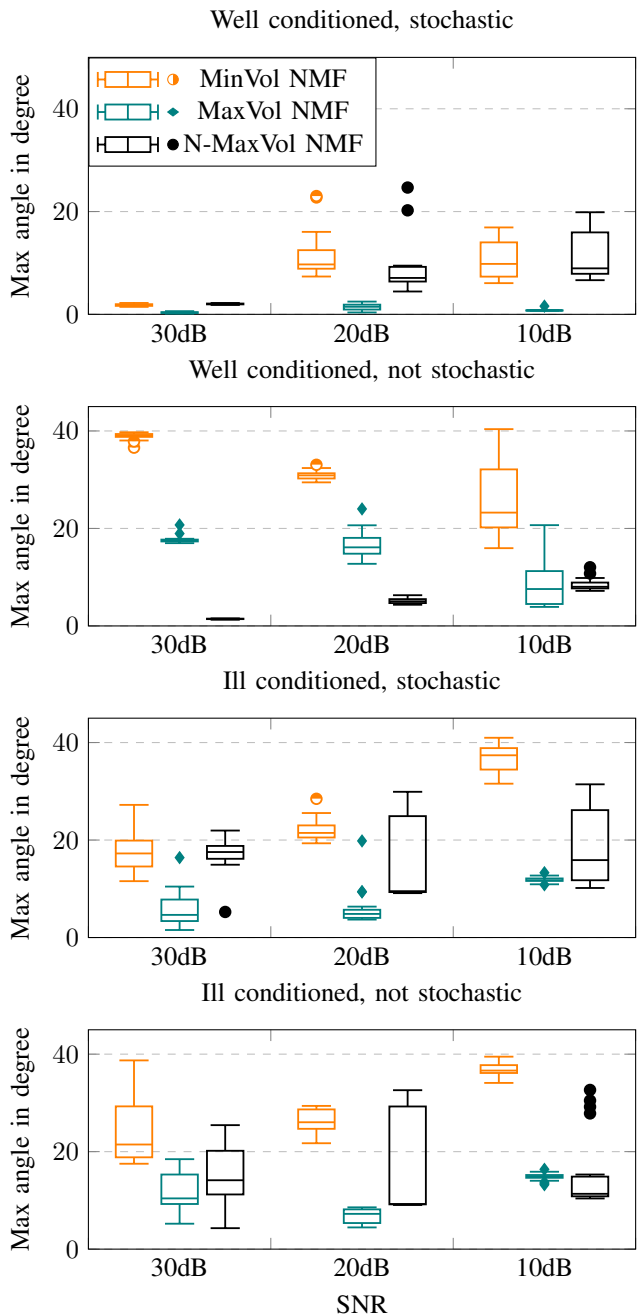
\begin{figure}[htbp!]
    \centering
    \begin{tikzpicture}
    \begin{axis}[
        title = {Well conditioned, stochastic},
        width=\linewidth,
        height=5cm,
        xticklabels={30dB,20dB,10dB},
        xmin=0,xmax=3,
        x tick label as interval,
        xtick={0,1,2,3},
        ylabel={Max angle in degree},
        ymin=0,ymax=50,
        ymajorgrids=true,
        grid style=dashed,
        boxplot={
            draw position={
                (1/(\the\nboxplots+1))
                + floor(\plotnumofactualtype/\the\nboxplots + 1e-9)
                + (1/(\the\nboxplots+1))*mod(\plotnumofactualtype,\the\nboxplots)
            },
            box extend=1/(\the\nboxplots+2)
        },
        legend style={
            at={(0,1)},anchor=north west,},]
    \addlegendimage{boxplot legend={orange,halfcircle*}}
    \addlegendimage{boxplot legend={teal,diamond*}}
    \addlegendimage{boxplot legend={black,*}}
    \foreach \col in {0,...,2}{
    \addplot[boxplot/draw direction=y,orange,mark=halfcircle*,line width=0.7pt] table[y index = \col]{maxvolnmf/figures/maxanglesminvol_perf.txt};
    \addplot[boxplot/draw direction=y,teal,mark=diamond*,line width=0.7pt] table[y index = \col]{maxvolnmf/figures/maxanglesvanilla_perf.txt};
    \addplot[boxplot/draw direction=y,black,line width=0.7pt] table[y index = \col]{maxvolnmf/figures/maxanglesmaxvol_perf.txt};
    };
    \legend{MinVol NMF, MaxVol NMF, N-MaxVol NMF}
    \end{axis}
    \end{tikzpicture}

    \begin{tikzpicture}
    \begin{axis}[
        title = {Well conditioned, not stochastic},
        width=\linewidth,
        height=4.5cm,
        xticklabels={30dB,20dB,10dB},
        xmin=0,xmax=3,
        x tick label as interval,
        xtick={0,1,2,3},
        ylabel={Max angle in degree},
        ymin=0,ymax=45,
        ymajorgrids=true,
        grid style=dashed,
        boxplot={
            draw position={
                (1/(\the\nboxplots+1))
                + floor(\plotnumofactualtype/\the\nboxplots + 1e-9)
                + (1/(\the\nboxplots+1))*mod(\plotnumofactualtype,\the\nboxplots)
            },
            box extend=1/(\the\nboxplots+2)
        },
        legend style={
            at={(0,1)},anchor=north west,},]
    \addlegendimage{boxplot legend={orange,halfcircle*}}
    \addlegendimage{boxplot legend={teal,diamond*}}
    \addlegendimage{boxplot legend={black,*}}
    \foreach \col in {0,...,2}{
    \addplot[boxplot/draw direction=y,orange,mark=halfcircle*,line width=0.7pt] table[y index = \col]{maxvolnmf/figures/maxanglesminvol_not_stoch.txt};
    \addplot[boxplot/draw direction=y,teal,mark=diamond*,line width=0.7pt] table[y index = \col]{maxvolnmf/figures/maxanglesvanilla_not_stoch.txt};
    \addplot[boxplot/draw direction=y,black,line width=0.7pt] table[y index = \col]{maxvolnmf/figures/maxanglesmaxvol_not_stoch.txt};
    };
    \end{axis}
    \end{tikzpicture}

    \begin{tikzpicture}
    \begin{axis}[
        title = {Ill conditioned, stochastic},
        width=\linewidth,
        height=4.5cm,
        xticklabels={30dB,20dB,10dB},
        xmin=0,xmax=3,
        x tick label as interval,
        xtick={0,1,2,3},
        ylabel={Max angle in degree},
        ymin=0,ymax=45,
        ymajorgrids=true,
        grid style=dashed,
        boxplot={
            draw position={
                (1/(\the\nboxplots+1))
                + floor(\plotnumofactualtype/\the\nboxplots + 1e-9)
                + (1/(\the\nboxplots+1))*mod(\plotnumofactualtype,\the\nboxplots)
            },
            box extend=1/(\the\nboxplots+2)
        },
        legend style={
            at={(0,1)},anchor=north west,},]
    \addlegendimage{boxplot legend={orange,halfcircle*}}
    \addlegendimage{boxplot legend={teal,diamond*}}
    \addlegendimage{boxplot legend={black,*}}
    \foreach \col in {0,...,2}{
    \addplot[boxplot/draw direction=y,orange,mark=halfcircle*,line width=0.7pt] table[y index = \col]{maxvolnmf/figures/maxanglesminvol_bad_cond.txt};
    \addplot[boxplot/draw direction=y,teal,mark=diamond*,line width=0.7pt] table[y index = \col]{maxvolnmf/figures/maxanglesvanilla_bad_cond.txt};
    \addplot[boxplot/draw direction=y,black,line width=0.7pt] table[y index = \col]{maxvolnmf/figures/maxanglesmaxvol_bad_cond.txt};
    };
    \end{axis}
    \end{tikzpicture}

    \begin{tikzpicture}
    \begin{axis}[
        title = {Ill conditioned, not stochastic},
        width=\linewidth,
        height=4.5cm,
        xticklabels={30dB,20dB,10dB},
        xmin=0,xmax=3,
        x tick label as interval,
        xtick={0,1,2,3},
        ylabel={Max angle in degree},xlabel={SNR},
        ymin=0,ymax=45,
        ymajorgrids=true,
        grid style=dashed,
        boxplot={
            draw position={
                (1/(\the\nboxplots+1))
                + floor(\plotnumofactualtype/\the\nboxplots + 1e-9)
                + (1/(\the\nboxplots+1))*mod(\plotnumofactualtype,\the\nboxplots)
            },
            box extend=1/(\the\nboxplots+2)
        },
        legend style={
            at={(0.0,1)},anchor=north west,},]
    \addlegendimage{boxplot legend={orange,halfcircle*}}
    \addlegendimage{boxplot legend={teal,diamond*}}
    \addlegendimage{boxplot legend={black,*}}
    \foreach \col in {0,...,2}{
    \addplot[boxplot/draw direction=y,orange,mark=halfcircle*,line width=0.7pt] table[y index = \col]{maxvolnmf/figures/maxanglesminvol_bad_cond_not_stoch.txt};
    \addplot[boxplot/draw direction=y,teal,mark=diamond*,line width=0.7pt] table[y index = \col]{maxvolnmf/figures/maxanglesvanilla_bad_cond_not_stoch.txt};
    \addplot[boxplot/draw direction=y,black,line width=0.7pt] table[y index = \col]{maxvolnmf/figures/maxanglesmaxvol_bad_cond_not_stoch.txt};
    };
    \end{axis}
    \end{tikzpicture}
    \caption{Box plots of the maximum angles between the ground truth and the estimated endmembers, depending on the model, the SNR and the configuration of $W$ and $H$.}
    \label{fig:synth_max_angles}
\end{figure}

\subsection{Hyperspectral images}

Let us now evaluate the performance of N-MaxVol NMF on famous hyperspectral datasets. Results can be compared with~\cite{zhu2017hyperspectral} where some ground-truths for a variety of known hyperspectral datasets are proposed. Even if these are called ground-truths, hyperspectral ground-truths do not exist except if the measurements are performed in a controlled environment. Consider the proposed abundance maps for Urban with four endmembers in~\cite{zhu2017hyperspectral}. Clearly, some trees are detected where in fact it should be a mixture of grass and soil. Some rooftops are also detected where it should be soil. Still, the author used as many a priori knowledge as possible to provide these abundance maps and endmembers that are probably close to reality. Our message here is that ground-truths for these hyperspectral datasets should be interpreted with caution. 
\revise{
\textbf{Choosing $\lambda$ and $\delta$ :} During our experiments, for N-MaxVol NMF, we almost never have to tweak the parameter $\delta$ and found that $\delta=0.5$ works well in most cases. The only exception is for the Jasper dataset with four endmembers, where the level of mixture required to increase $\delta$ to 1 in order to obtain satisfying results. The parameter $\lambda$ also requires minimal tuning effort as it is automatically tuned using the approach proposed in~\cite{nguyen2024towards}. In most cases a value of $\lambda$ between 0.5 and 1 works well.  
}

On Moffett and on Samson, our model clearly outperforms MinVol NMF and MaxVol NMF, see \Cref{fig:moffett_nmaxvol_1_05,fig:samson_nmaxvol}. Water, soil and tree are correctly separated and their spectral signatures are very close to the expected ones in \cite{zhu2017hyperspectral}. The visual results for MinVol and MaxVol NMF for all data sets are provided in the Supplementary Material.  

\begin{figure}[htbp!]
	\hfill\fbox{\includegraphics[width=0.93\linewidth]{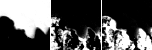}}

	\begin{tikzpicture}
        \begin{groupplot}[
            group style={
                group name=my plots,
                group size= 3 by 1,
                x descriptions at=edge bottom,
                y descriptions at=edge left,
                horizontal sep=0.0cm,
                vertical sep=0.cm,
            },
			scale only axis,width=1/3*0.93\linewidth,no markers,grid=major,xlabel={},ylabel={},tick label style={/pgf/number format/fixed,font=\tiny},ymin=0,ymax=0.3,ytick={0,0.1,...,0.5},
        ]
        \nextgroupplot[]
            \addplot [line width = 1pt,blue] table [x expr=\coordindex+1, y index=0] {maxvolnmf/figures/moffett_nmaxvol_1_05.txt};
        \nextgroupplot[]
            \addplot [line width = 1pt,red] table [x expr=\coordindex+1, y index=1] {maxvolnmf/figures/moffett_nmaxvol_1_05.txt};
        \nextgroupplot[]
            \addplot [line width = 1pt] table [x expr=\coordindex+1, y index=2] {maxvolnmf/figures/moffett_nmaxvol_1_05.txt};
        \end{groupplot}
	\end{tikzpicture}
    \caption[Abundance maps and endmembers by N-MaxVol NMF on Moffett]{Abundance maps and endmembers ({\color{blue}water}, {\color{red}tree} and soil) by N-MaxVol NMF on Moffett, with $\lambda=1$ and $\delta=0.5$.}
    \label{fig:moffett_nmaxvol_1_05}
\end{figure}
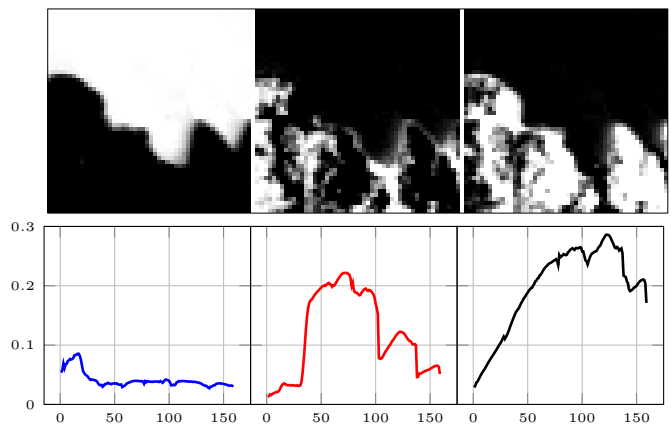

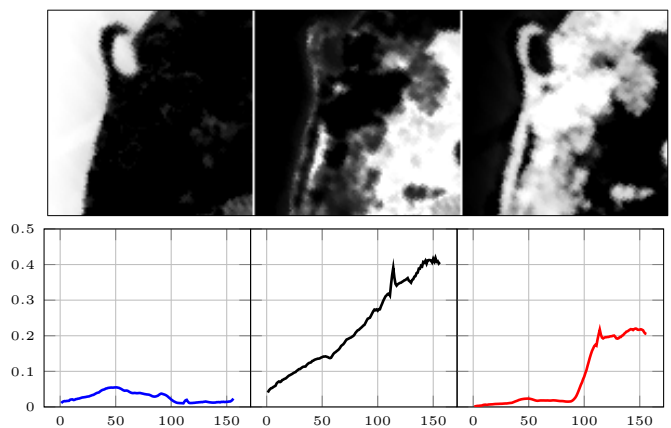
\begin{figure}[htbp!]
	\hfill\fbox{\includegraphics[width=0.93\linewidth]{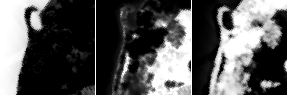}}

	\begin{tikzpicture}
        \begin{groupplot}[
            group style={
                group name=my plots,
                group size= 3 by 1,
                x descriptions at=edge bottom,
                y descriptions at=edge left,
                horizontal sep=0.0cm,
                vertical sep=0.cm,
            },
			scale only axis,width=1/3*0.93\linewidth,no markers,grid=major,xlabel={},ylabel={},tick label style={/pgf/number format/fixed,font=\tiny},ymin=0,ymax=0.5,ytick={0,0.1,...,0.5},
        ]
        \nextgroupplot
            \addplot [line width = 1pt,blue] table [x expr=\coordindex+1, y index=0] {maxvolnmf/figures/samson_nmaxvol.txt};
        \nextgroupplot
            \addplot [line width = 1pt] table [x expr=\coordindex+1, y index=1] {maxvolnmf/figures/samson_nmaxvol.txt};
        \nextgroupplot
            \addplot [line width = 1pt,red] table [x expr=\coordindex+1, y index=2] {maxvolnmf/figures/samson_nmaxvol.txt};
        \end{groupplot}
	\end{tikzpicture}
    \caption[Abundance maps and endmembers by N-MaxVol NMF on Samson]{Abundance maps and endmembers ({\color{blue}water}, soil and {\color{red}tree}) by N-MaxVol NMF on Samson, with $\lambda=1$ and $\delta=0.5$.}
    \label{fig:samson_nmaxvol}
\end{figure}


\subsubsection{Samson hyperspectral image}

Let us consider the Samson with $r=6$, although Samson contains mostly 3 endembers.  
With MinVol NMF, the 3 virtual endmembers are brought to zero by properly tuning $\lambda$ and $\delta$. 
On the other hand, using $r=6$ allows MaxVol NMF  to learn more spectral varieties; see \Cref{fig:samson_r6_smaxvol_05_05}. 
We see three different kinds of tree and two different kinds of soil. We can then add together the rows of $H$ that correspond to varieties of the same endmembers. The resulting merged abundance maps are shown on \Cref{fig:grouped_samson}. The abundance maps on \Cref{fig:grouped_samson} and \Cref{fig:samson_nmaxvol} are very close to each other. Actually, results with $r=6$ are more satisfying for the water unmixing. On \Cref{fig:samson_nmaxvol}, some small artifacts of false-positives can be seen, especially in the bottom right corner of the abundance map corresponding to water. These artifacts are not visible on \Cref{fig:grouped_samson}. 
This illustrates that MaxVol NMF allows us to increase the number of endmembers in order to improve the results in a controlled manner.

\subsubsection{Urban  hyperspectral image} 

The Urban datasets contains 162 spectral bands, and images of size 307 $\times$ 307. It is the aerial photo of a Walmart in Texas. 
 This dataset is particularly insightful because it is known for having meaningful unmixing results for $r=4,5\text{ and }6$~\cite{zhu2017hyperspectral}.
 Results are displayed on \Cref{fig:urban_nmaxvol}. With $r=4$, we have roof, grass, a combination of asphalt and soil, and tree. With $r=5$, the distinction is being made between asphalt and dirt. With $r=6$, the distinction is being made between grass and dry grass. Typical ground-truths with $r=6$ rather suggest a distinction between two kinds of roofs, instead of grass and dry grass. Here our model \revise{proposes} another interpretation for $r=6$ which  makes sense.

\subsubsection{Jasper hyperspectral image}

The last experiment is on the Jasper dataset, where ground-truths suggest four endmembers: tree, water, soil and road. Unmixing algorithms often struggle to correctly separate water and road on Jasper. 
\Cref{fig:jasper4} shows that N-MaxVol NMF achieves not ideal but nonetheless decent results. The issue with our model here is that in order to improve the distinction between water and road, $\lambda$ should be increased. However, 
there are many areas where tree and soil are heavily mixed, and increasing $\lambda$ will associate these pixels with single endmember. 
One way to circumvent this issue is to increase the rank. Results with $r=5$ are displayed on \Cref{fig:jasper5}. The additional endmember is in fact a combination of tree and soil. With this trick, water and road are properly identified without compromising the quality of the other endmembers.

\revise{Overall, across random initializations, N-MaxVol NMF is very consistent for the simpler datasets (Moffett, Samson, and Urban). More complicated datasets, such as Jasper, can provide slightly different results depending on the initialization and the choice of $\lambda$ and $\delta$.}

\begin{figure}[htbp!]
	\hfill\fbox{\includegraphics[width=0.93\linewidth]{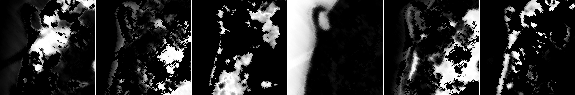}}

    \begin{tikzpicture}
        \begin{groupplot}[
            group style={
                group name=my plots,
                group size= 6 by 1,
                x descriptions at=edge bottom,
                y descriptions at=edge left,
                horizontal sep=0.0cm,
                vertical sep=0.cm,
            },
			scale only axis,width=1/6*0.93\linewidth,no markers,grid=major,xlabel={},ylabel={},tick label style={/pgf/number format/fixed,font=\tiny},ymin=0,ymax=0.2,ytick={0,0.1,...,0.5},
        ]
        \nextgroupplot
            \addplot [line width = 1pt] table [x expr=\coordindex+1, y index=0] {maxvolnmf/figures/samson_r6_smaxvol_05_05.txt};
        \nextgroupplot
            \addplot [line width = 1pt,red] table [x expr=\coordindex+1, y index=1] {maxvolnmf/figures/samson_r6_smaxvol_05_05.txt};
        \nextgroupplot
            \addplot [line width = 1pt] table [x expr=\coordindex+1, y index=2] {maxvolnmf/figures/samson_r6_smaxvol_05_05.txt};
        \nextgroupplot
            \addplot [line width = 1pt,blue] table [x expr=\coordindex+1, y index=3] {maxvolnmf/figures/samson_r6_smaxvol_05_05.txt};
        \nextgroupplot
            \addplot [line width = 1pt,red] table [x expr=\coordindex+1, y index=4] {maxvolnmf/figures/samson_r6_smaxvol_05_05.txt};
        \nextgroupplot
            \addplot [line width = 1pt] table [x expr=\coordindex+1, y index=5] {maxvolnmf/figures/samson_r6_smaxvol_05_05.txt};
        \end{groupplot}
	\end{tikzpicture}
	\caption[Abundance maps and endmembers by N-MaxVol NMF on Samson with $r=6$]{Abundance maps and endmembers (tree, {\color{red}soil} and {\color{blue}water}) by N-MaxVol NMF with $r=6$ on Samson, with $\lambda=0.5$ and $\delta=0.5$.}
    \label{fig:samson_r6_smaxvol_05_05}
\end{figure}

\begin{figure}[htbp!]
	\hfill\fbox{\includegraphics[width=0.93\linewidth]{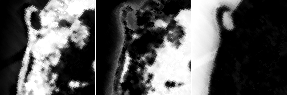}}

    \begin{tikzpicture}
        \begin{groupplot}[
            group style={
                group name=my plots,
                group size= 3 by 1,
                x descriptions at=edge bottom,
                y descriptions at=edge left,
                horizontal sep=0.0cm,
                vertical sep=0.cm,
            },
			scale only axis,width=1/3*0.93\linewidth,no markers,grid=major,xlabel={},ylabel={},tick label style={/pgf/number format/fixed,font=\tiny},ymin=0,ymax=0.2,ytick={0,0.1,...,0.5},
        ]
        \nextgroupplot
            \addplot [line width = 1pt] table [x expr=\coordindex+1, y index=0] {maxvolnmf/figures/samson_r6_smaxvol_05_05.txt};
            \addplot [line width = 1pt,dashed] table [x expr=\coordindex+1, y index=2] {maxvolnmf/figures/samson_r6_smaxvol_05_05.txt};
            \addplot [line width = 1pt,dotted] table [x expr=\coordindex+1, y index=5] {maxvolnmf/figures/samson_r6_smaxvol_05_05.txt};
        \nextgroupplot
            \addplot [line width = 1pt,red] table [x expr=\coordindex+1, y index=1] {maxvolnmf/figures/samson_r6_smaxvol_05_05.txt};
            \addplot [line width = 1pt,red,dashed] table [x expr=\coordindex+1, y index=4] {maxvolnmf/figures/samson_r6_smaxvol_05_05.txt};
        \nextgroupplot
            \addplot [line width = 1pt,blue] table [x expr=\coordindex+1, y index=3] {maxvolnmf/figures/samson_r6_smaxvol_05_05.txt};
        \end{groupplot}
	\end{tikzpicture}
	\caption[Abundance maps grouped by endmember varieties and endmembers by N-MaxVol NMF with $r=6$ on Samson]{Abundance maps grouped by endmember varieties and endmembers (tree, {\color{red}soil} and {\color{blue}water}) by N-MaxVol NMF with $r=6$ on Samson, with $\lambda=0.5$ and $\delta=0.5$.}
    \label{fig:grouped_samson}
\end{figure}

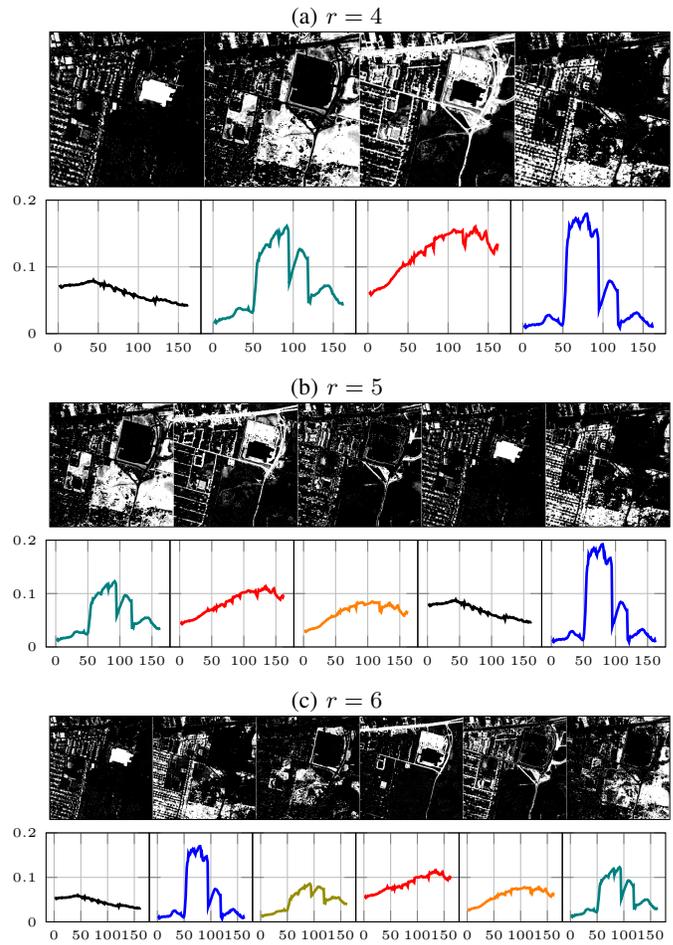
\begin{figure}[htbp!]
    \begin{subfigure}{\linewidth}
        \caption{$r=4$}
        \hfill\fbox{\includegraphics[width=0.93\linewidth]{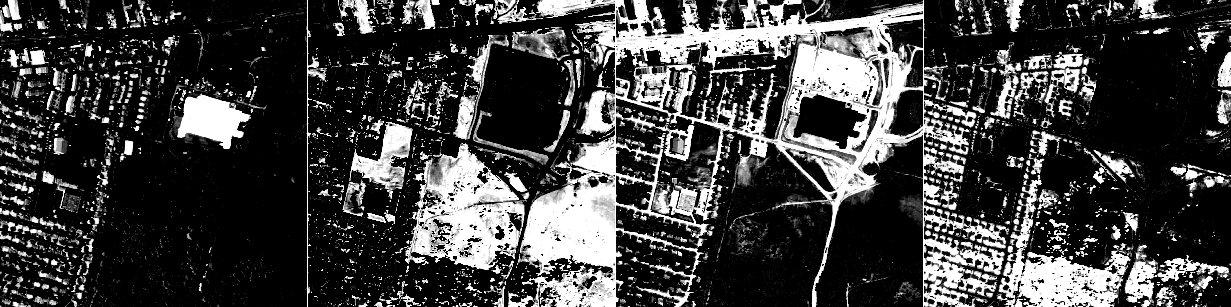}}
    
        \begin{tikzpicture}
            \begin{groupplot}[
                group style={
                    group name=my plots,
                    group size= 4 by 1,
                    x descriptions at=edge bottom,
                    y descriptions at=edge left,
                    horizontal sep=0.0cm,
                    vertical sep=0.cm,
                },
                scale only axis,width=1/4*0.93\linewidth,no markers,grid=major,xlabel={},ylabel={},tick label style={/pgf/number format/fixed,font=\tiny},ymin=0,ymax=0.2,ytick={0,0.1,...,0.5},
            ]
            \nextgroupplot
                \addplot [line width = 1pt] table [x expr=\coordindex+1, y index=0] {maxvolnmf/figures/urban4_smaxvol_05_05.txt};
            \nextgroupplot
                \addplot [line width = 1pt,teal] table [x expr=\coordindex+1, y index=1] {maxvolnmf/figures/urban4_smaxvol_05_05.txt};
            \nextgroupplot
                \addplot [line width = 1pt,red] table [x expr=\coordindex+1, y index=2] {maxvolnmf/figures/urban4_smaxvol_05_05.txt};
            \nextgroupplot
                \addplot [line width = 1pt,blue] table [x expr=\coordindex+1, y index=3] {maxvolnmf/figures/urban4_smaxvol_05_05.txt};
            \end{groupplot}
        \end{tikzpicture}
    \end{subfigure}

    \begin{subfigure}{\linewidth}
        \caption{$r=5$}
        \hfill\fbox{\includegraphics[width=0.93\linewidth]{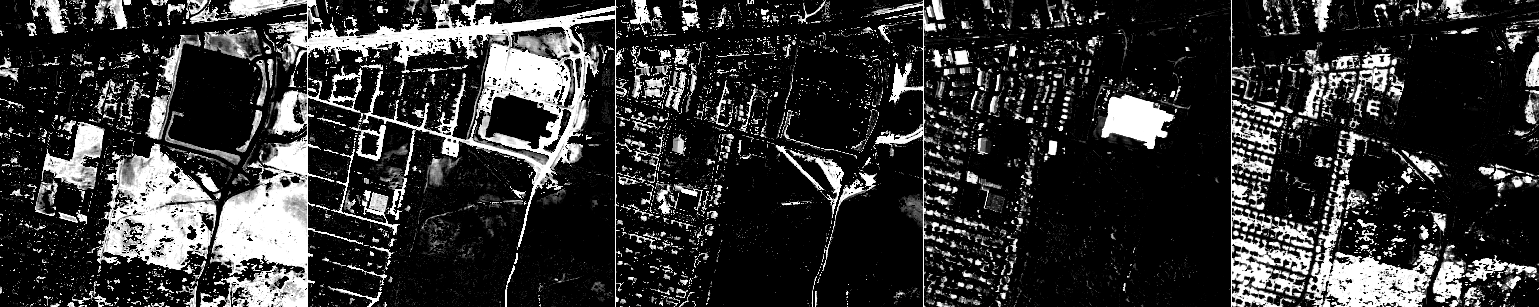}}
    
        \begin{tikzpicture}
            \begin{groupplot}[
                group style={
                    group name=my plots,
                    group size= 5 by 1,
                    x descriptions at=edge bottom,
                    y descriptions at=edge left,
                    horizontal sep=0.0cm,
                    vertical sep=0.cm,
                },
                scale only axis,width=1/5*0.93\linewidth,no markers,grid=major,xlabel={},ylabel={},tick label style={/pgf/number format/fixed,font=\tiny},ymin=0,ymax=0.2,ytick={0,0.1,...,0.5},
            ]
            \nextgroupplot
                \addplot [line width = 1pt,teal] table [x expr=\coordindex+1, y index=0] {maxvolnmf/figures/urban5_smaxvol_05_05.txt};
            \nextgroupplot
                \addplot [line width = 1pt,red] table [x expr=\coordindex+1, y index=1] {maxvolnmf/figures/urban5_smaxvol_05_05.txt};
            \nextgroupplot
                \addplot [line width = 1pt,orange] table [x expr=\coordindex+1, y index=2] {maxvolnmf/figures/urban5_smaxvol_05_05.txt};
            \nextgroupplot
                \addplot [line width = 1pt] table [x expr=\coordindex+1, y index=3] {maxvolnmf/figures/urban5_smaxvol_05_05.txt};
            \nextgroupplot
                \addplot [line width = 1pt,blue] table [x expr=\coordindex+1, y index=4] {maxvolnmf/figures/urban5_smaxvol_05_05.txt};
            \end{groupplot}
        \end{tikzpicture}
    \end{subfigure}

    \begin{subfigure}{\linewidth}
        \caption{$r=6$}
        \hfill\fbox{\includegraphics[width=0.93\linewidth]{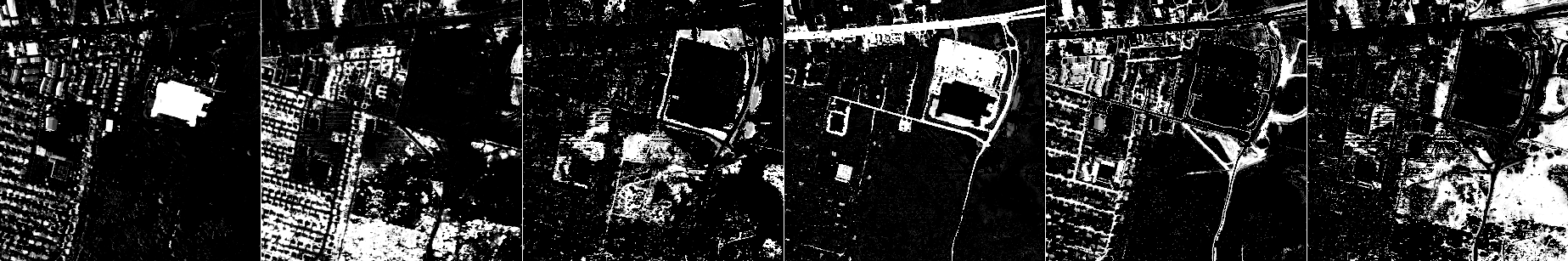}}
    
        \begin{tikzpicture}
            \begin{groupplot}[
                group style={
                    group name=my plots,
                    group size= 6 by 1,
                    x descriptions at=edge bottom,
                    y descriptions at=edge left,
                    horizontal sep=0.0cm,
                    vertical sep=0.cm,
                },
                scale only axis,width=1/6*0.93\linewidth,no markers,grid=major,xlabel={},ylabel={},tick label style={/pgf/number format/fixed,font=\tiny},ymin=0,ymax=0.2,ytick={0,0.1,...,0.5},
            ]
            \nextgroupplot
                \addplot [line width = 1pt] table [x expr=\coordindex+1, y index=0] {maxvolnmf/figures/urban6_smaxvol_05_05.txt};
            \nextgroupplot
                \addplot [line width = 1pt,blue] table [x expr=\coordindex+1, y index=1] {maxvolnmf/figures/urban6_smaxvol_05_05.txt};
            \nextgroupplot
                \addplot [line width = 1pt,olive] table [x expr=\coordindex+1, y index=2] {maxvolnmf/figures/urban6_smaxvol_05_05.txt};
            \nextgroupplot
                \addplot [line width = 1pt,red] table [x expr=\coordindex+1, y index=3] {maxvolnmf/figures/urban6_smaxvol_05_05.txt};
            \nextgroupplot
                \addplot [line width = 1pt,orange] table [x expr=\coordindex+1, y index=4] {maxvolnmf/figures/urban6_smaxvol_05_05.txt};
            \nextgroupplot
                \addplot [line width = 1pt,teal] table [x expr=\coordindex+1, y index=5] {maxvolnmf/figures/urban6_smaxvol_05_05.txt};
            \end{groupplot}
        \end{tikzpicture}
    \end{subfigure}
	\caption[Abundance maps and endmembers by N-MaxVol NMF on Urban depending on $r$]{Abundance maps and endmembers (roof, {\color{blue}tree}, {\color{olive}dry grass}, {\color{red}asphalt}, {\color{orange}soil}, {\color{teal}grass}) by N-MaxVol NMF on Urban, with $\lambda=0.5$ and $\delta=0.5$, depending on $r$.}
    \label{fig:urban_nmaxvol}
\end{figure}



\begin{figure}[htbp!]
	\hfill\fbox{\includegraphics[width=0.93\linewidth]{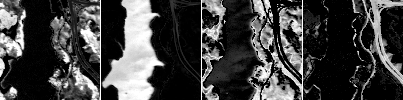}}

    \begin{tikzpicture}
        \begin{groupplot}[
            group style={
                group name=my plots,
                group size= 4 by 1,
                x descriptions at=edge bottom,
                y descriptions at=edge left,
                horizontal sep=0.0cm,
                vertical sep=0.cm,
            },
			scale only axis,width=1/4*0.93\linewidth,no markers,grid=major,xlabel={},ylabel={},tick label style={/pgf/number format/fixed,font=\tiny},ymin=0,ymax=0.3,ytick={0,0.1,...,0.5},
        ]
        \nextgroupplot
            \addplot [line width = 1pt] table [x expr=\coordindex+1, y index=0] {maxvolnmf/figures/jasper4.txt};
        \nextgroupplot
            \addplot [line width = 1pt,blue] table [x expr=\coordindex+1, y index=1] {maxvolnmf/figures/jasper4.txt};
        \nextgroupplot
            \addplot [line width = 1pt,red] table [x expr=\coordindex+1, y index=2] {maxvolnmf/figures/jasper4.txt};
        \nextgroupplot
            \addplot [line width = 1pt,teal] table [x expr=\coordindex+1, y index=3] {maxvolnmf/figures/jasper4.txt};
        \end{groupplot}
	\end{tikzpicture}
	\caption[Abundance maps and endmembers by N-MaxVol NMF with $r=4$ on Jasper]{Abundance maps and endmembers (tree, {\color{blue}water}, {\color{red}soil}, {\color{teal}road}) by N-MaxVol NMF with $r=4$ on Jasper, with $\lambda=2$ and $\delta=1$.}
    \label{fig:jasper4}
\end{figure}
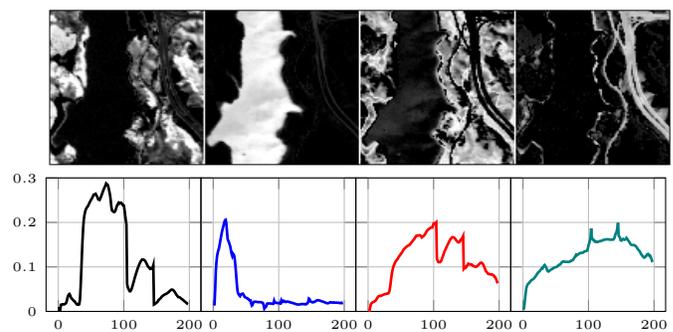

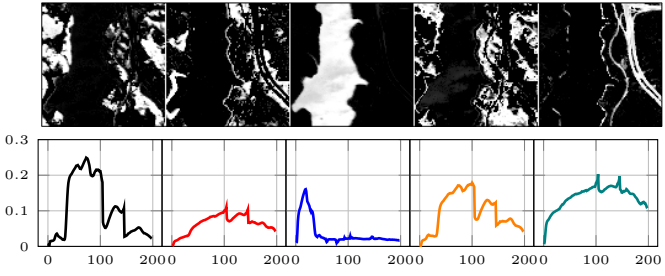
\begin{figure}[htbp!]
	\hfill\fbox{\includegraphics[width=0.93\linewidth]{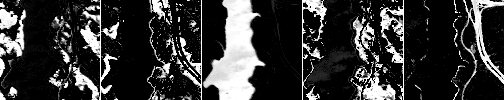}}

    \begin{tikzpicture}
        \begin{groupplot}[
            group style={
                group name=my plots,
                group size= 5 by 1,
                x descriptions at=edge bottom,
                y descriptions at=edge left,
                horizontal sep=0.0cm,
                vertical sep=0.cm,
            },
			scale only axis,width=1/5*0.93\linewidth,no markers,grid=major,xlabel={},ylabel={},tick label style={/pgf/number format/fixed,font=\tiny},ymin=0,ymax=0.3,ytick={0,0.1,...,0.5},
        ]
        \nextgroupplot
            \addplot [line width = 1pt] table [x expr=\coordindex+1, y index=0] {maxvolnmf/figures/jasper5.txt};
        \nextgroupplot
            \addplot [line width = 1pt,red] table [x expr=\coordindex+1, y index=1] {maxvolnmf/figures/jasper5.txt};
        \nextgroupplot
            \addplot [line width = 1pt,blue] table [x expr=\coordindex+1, y index=2] {maxvolnmf/figures/jasper5.txt};
        \nextgroupplot
            \addplot [line width = 1pt,orange] table [x expr=\coordindex+1, y index=3] {maxvolnmf/figures/jasper5.txt};
        \nextgroupplot
            \addplot [line width = 1pt,teal] table [x expr=\coordindex+1, y index=4] {maxvolnmf/figures/jasper5.txt};
        \end{groupplot}
	\end{tikzpicture}
	\caption[Abundance maps and endmembers by N-MaxVol NMF with $r=5$ on Jasper]{Abundance maps and endmembers (tree, {\color{red}soil}, {\color{blue}water}, {\color{orange}tree+soil}, {\color{teal}road}) by N-MaxVol NMF with $r=5$ on Jasper, with $\lambda=0.5$ and $\delta=0.5$.}
    \label{fig:jasper5}
\end{figure}


\section{Conclusion}\label{sec:conclusion}

In this paper, we first explored MaxVol NMF, a dual version of MinVol NMF, where the volume of $H$ is maximized instead of the volume of $W$ being minimized. 
In the noiseless case, MaxVol NMF is identifiable under the same conditions as MinVol NMF. 
In noisy settings, MaxVol NMF alleviates two important issues of MinVol NMF: the generation of rank-deficient solutions, and the difficulty to control the sparsity of $H$. 
We proposed two different algorithms to solve MaxVol NMF.  A drawback of MaxVol NMF is that it tends to generate clusters of the same size. This motivated us to introduce normalized MaxVol NMF (N-MaxVol NMF), a variant where the volume of the row-wise normalized $H$ factor is maximized. This model creates a continuum between NMF and ONMF, and leads to better separations than MinVol NMF on synthetic and hyperspectral unmixing. 

Further research directions include the study of the  identifiability of N-MaxVol NMF, and the use of N-MaxVol NMF on other types of data, e.g., document data sets or audio signals.  

\revise{
\section*{Acknowledgments}
We are grateful to the anonymous reviewers who carefully read the manuscript and provided feedback that helped us improve our paper.
}

\appendices

\revise{
\section{Beware of the sum-to-one constraint on the rows of $H$} \label{app:SSCrownorm} 

In this paper, we do not consider the sum-to-one constraint on the rows of $H$. Let us explain why. 
For an NMF, $X = WH$, sum-to-one constraint on the rows of $H$ can be made without loss of generality by the scaling degree of freedom with the columns of $W$, as opposed to the sum-to-one constraint on the columns of $H$.  
Moreover, MinVol NMF with this normalization is also identifiable: For $X = W'H'$ with $W'$ full rank, $H'1_n = 1_r$ and $H'$ satisfying the SSC, any optimal solution of  
\begin{equation} \label{eq:NMFrowsumtoone}
\min_{W \geq 0, H \geq 0} \det(W^\top W) \text{ such that } H 1_n = 1_r  
\end{equation} 
recovers $(W',H')$ up to permutation~\cite{fu2018identifiability}. 

Where is the issue with this normalization? 
The issue is that \emph{imposing the sum-to-one constraint on the rows of $H$ in an NMF can  destroy the SSC}. In other terms, row-scaling of a matrix might affect the SSC. 
In mathematical terms, let $X = WH$ be an NMF, and let $H_n$ be the scaled version of $H$ with sum-to-one rows 
and $W_n$ be scaled accordingly, so that $X = W_n H_n$. 
There is no guarantee that $H_n$ will satisfy the SSC when $H$ does, and vice versa.  
To the best of our knowledge, this subtle issue has never been highlighted in the literature, although it plays a crucial role in practice in the identifiability of MinVol NMF. (Note that scaling the columns of $H$ does not affect the SSC, since the SSC relies on properties of  $\cone(H)$ which is invariant to column scaling of $H$.) 

Let us illustrate this issue with an example and consider 
\begin{equation} \label{eq:HsscExample}
H = \left( \begin{array}{ccccc}
    1 & 0 & 0.4 & 0   & 0.1 \\
    0 & 1 & 0   & 0.4 & 0.1 \\ 
    0 & 0 & 0.6 & 0.6 & 0.8  
\end{array}
\right)
\end{equation}
which satisfies the SSC; this can be checked with the code from~\cite{gillis2024checking}, see also  Figure~\ref{fig:HsscExample} for an illustration.  

\begin{figure}[ht!]
    \centering 
    \tikzsetnextfilename{SSC-H}
\begin{tikzpicture}[scale=1]

\coordinate (e1) at (0,0);
\coordinate (e2) at (4,0);
\coordinate (e3) at (2,{2*sqrt(3)});

\coordinate (p1) at (1.2,{1.2*sqrt(3)});
\coordinate (p2) at (2.8,{1.2*sqrt(3)});

\coordinate (q) at (2,{1.6*sqrt(3)});

\fill[red!25]
    (e1) -- (e2) -- (p2) -- (q) -- (p1) -- cycle;

\draw[blue, thick, dashed] (e1) -- (e2) -- (e3) -- cycle;

\draw[red, thick]
    (e1) -- (e2) -- (p2) -- (q) -- (p1) -- cycle;

\node[blue, below left]  at (e1) {$e_1=(1,0,0)$};
\node[blue, below right] at (e2) {$e_2=(0,1,0)$};
\node[blue, above]       at (e3) {$e_3=(0,0,1)$};

\draw[green!60!black, thick, dotted]
    (2,{2*sqrt(3)/3}) circle ({2*sqrt(3)/3});

\fill[red] (e1) circle (2.5pt);
\fill[red] (e2) circle (2.5pt);
\fill[red] (p1) circle (2.5pt);
\fill[red] (p2) circle (2.5pt);
\fill[red] (q)  circle (2.5pt);


\node[red] at (2,-0.5) {$\operatorname{cone}(H) \cap \Delta^3$};
\node[green!60!black] at (2,{2*sqrt(3)/3}) {$\mathcal{C} \cap \Delta^3$};
\node[blue] at (3,3) {$\Delta^3$};

\end{tikzpicture}
    \caption{Illustration that $H$ from~\eqref{eq:HsscExample} satisfies the SSC. \label{fig:HsscExample} }
\end{figure}

The sum-to-one row-normalized version of $H$ is 
\begin{equation} \label{eq:HsscExample2}
H_n = \left( \begin{array}{ccccc}
2/3 & 0 & 4/15 & 0 & 1/15 \\
0 & 2/3 & 0 & 4/15 & 1/15 \\
0 & 0 & 3/10 & 3/10 & 2/5 
  \end{array}
\right), 
\end{equation}     
which does not satisfy the SSC; see Figure~\ref{fig:HsscExample2}. The reason is that the third row of $H$ has a larger sum and row normalization downweights the third endmember (that is, the third column of~$W$). 
\begin{figure}[ht!] 
\tikzsetnextfilename{SSC-Hn}
\begin{tikzpicture}[scale=1]

\coordinate (e1) at (0,0);
\coordinate (e2) at (4,0);
\coordinate (e3) at (2,{2*sqrt(3)});

\coordinate (p1) at ({18/17},{18*sqrt(3)/17});
\coordinate (p2) at ({50/17},{18*sqrt(3)/17});

\coordinate (q) at (2,{3*sqrt(3)/2});

\fill[red!25]
    (e1) -- (e2) -- (p2) -- (q) -- (p1) -- cycle;

\draw[blue, thick, dashed] (e1) -- (e2) -- (e3) -- cycle;

\draw[red, thick]
    (e1) -- (e2) -- (p2) -- (q) -- (p1) -- cycle;

\node[blue, below left]  at (e1) {$e_1=(1,0,0)$};
\node[blue, below right] at (e2) {$e_2=(0,1,0)$};
\node[blue, above]       at (e3) {$e_3=(0,0,1)$};

\draw[green!60!black, thick, dotted]
    (2,{2*sqrt(3)/3}) circle ({2*sqrt(3)/3});

\fill[red] (e1) circle (2.5pt);
\fill[red] (e2) circle (2.5pt);
\fill[red] (p1) circle (2.5pt);
\fill[red] (p2) circle (2.5pt);
\fill[red] (q)  circle (2.5pt);

\node[red] at (2,-0.5) {$\operatorname{cone}(H_n) \cap \Delta^3$};
\node[green!60!black] at (2,{2*sqrt(3)/3}) {$\mathcal{C} \cap \Delta^3$};
\node[blue] at (3,3.2) {$\Delta^3$};

\end{tikzpicture}

    \caption{Illustration that $H_n$ from~\eqref{eq:HsscExample2} does not satisfy the SSC because $\cone(H_n)$ does not contain $\mathcal C$. \label{fig:HsscExample2} }
\end{figure}

Even worse, let us construct the matrix $H_d$ by taking the matrix $H$ in \eqref{eq:HsscExample} but duplicating the last column 10 times and then scaling the rows to sum to one, which gives 
\[
H_d = \left( 
\begin{array}{ccccccc}
5/12   &        0       &       1/6       &     0        &      1/24       &    \dots     &      1/24 \\ 0      &        5/12       &    0   &           1/6    &        1/24     &      \dots       &   1/24\\  0  &            0          &    3/46   &        3/46     &      2/23     &      \dots       &    2/23 
\end{array}
\right). 
\] 
Because the third row of $H$ with the duplicated column has an even larger sum (namely 9.2, as opposed 2.4 for the first two rows), the third endmember is even more downweighted and the corresponding scaled matrix $H_d$ is very far from satisfying the SSC; see Figure~\ref{fig:HsscExample3} for an illustration.

\begin{figure}[ht!]

\tikzsetnextfilename{SSC-Hd}
\begin{tikzpicture}[scale=1]

\coordinate (e1) at (0,0);
\coordinate (e2) at (4,0);
\coordinate (e3) at (2,{2*sqrt(3)});

\coordinate (p1) at ({9/16},{9*sqrt(3)/16});
\coordinate (p2) at ({55/16},{9*sqrt(3)/16});

\coordinate (q) at (2,{48*sqrt(3)/47});

\coordinate (b1) at (39/31,{39/31*sqrt(3)});
\coordinate (b2) at (5.25,0);

\fill[red!25]
    (e1) -- (e2) -- (p2) -- (q) -- (p1) -- cycle;

\draw[blue, thick, dashed]
    (e1) -- (e2) -- (e3) -- cycle;

\draw[red, thick]
    (e1) -- (e2) -- (p2) -- (q) -- (p1) -- cycle;

 \draw[gray, very thick]
    (e1) -- (b1) -- (b2) -- cycle;

\node[blue, below left]  at (e1) {$e_1=(1,0,0)$};
\node[blue, below right] at (e2) {$e_2=(0,1,0)$};
\node[blue, above]       at (e3) {$e_3=(0,0,1)$};

\draw[green!60!black, thick, dotted]
    (2,{2*sqrt(3)/3}) circle ({2*sqrt(3)/3});

\fill[red] (e1) circle (2.5pt);
\fill[red] (e2) circle (2.5pt);
\fill[red] (p1) circle (2.5pt);
\fill[red] (p2) circle (2.5pt);
\fill[red] (q)  circle (2.5pt);

\node[red] at (2,-0.5)
    {$\operatorname{cone}(H_d) \cap \Delta^3$};

\node[green!60!black] at (2,{2*sqrt(3)/3})
    {$\mathcal{C} \cap \Delta^3$};

\node[blue] at (3,3) {$\Delta^3$};

\end{tikzpicture}

    \caption{Illustration that $H_d$, the scaled version with sum-to-one rows  of $H$ in \eqref{eq:HsscExample}  where the last column is duplicated 10 times, is very far from satisfying the SSC. The gray triangle contains the columns of $H_d$ and has smaller area than the blue triangle. 
    \label{fig:HsscExample3}}
\end{figure}

Moreover, in this case, a triangle with smaller area contains\footnote{More precisely, the cone generated by a triangle of smaller aperture contains $\cone(H_d)$.} the columns of $H_d$, as shown on Figure~\ref{fig:HsscExample3}. This allows us to construct an NMF problem where MinVol NMF with the sum-to-one constraint on the rows fails: Take $H_d$ as above, 
\[ 
W_d = \left( \begin{array}{ccc}
    1 & 1 & 0     \\
    0 & 1 & 0   \\ 
    0 & 0 & 1   
\end{array}
\right), 
\] 
and let $X = W_d H_d$. After some calculations, the solution corresponding to the triangle with smaller area is 
\[ 
W_d' = \left( \begin{array}{ccc}
25/39   &   10/13   &    23/39    \\ 
  0   & 1   &    0       \\ 
  0   &    0     &      1 
\end{array}
\right), 
\] 
whose corresponding $H_d'$ such that $X = W_d' H_d'$ and \mbox{$H_d' 1_n = 1_r$} is given by 
\[
H_d' = \left( 
\begin{array}{ccccccc}
13/20 & 3/20 & 1/5 & 0 & 0 & \dots & 0 \\
0 & 5/12 & 0 & 1/6 & 1/24 & \dots & 1/24 \\
0 & 0 & 3/46 & 3/46 & 2/23 & \dots & 2/23
\end{array}
\right), 
\] 
where $\det(W_d'^\top W_d') =  0.4109 < \det(W_d^\top W_d) = 1$. 
Therefore, for the matrix $X = W_d H_d$, MinVol NMF under the sum-to-one constraint on the rows~\eqref{eq:NMFrowsumtoone} fails to recover $(W_d,H_d)$ although $H_d$ satisfies the SSC, and hence MinVol NMF under the sum-to-one constraint on the columns works. 
Note that $H_d'$ is also far from satisfying the SSC, see Figure~\ref{fig:HsscExample4}, which prevents us to guarantee identifiability of the NMF $X = W_d' H_d'$ when solving~\eqref{eq:NMFrowsumtoone}. 
\begin{figure}
\tikzsetnextfilename{SSC-Hdp}
\begin{tikzpicture}[scale=1]

\coordinate (e1) at (0,0);
\coordinate (e2) at (4,0);
\coordinate (e3) at (2,{2*sqrt(3)});

%
%

\coordinate (h1) at (0,0);

\coordinate (h2) at ({50/17},0);

\coordinate (h3) at ({30/61},{30*sqrt(3)/61});

\coordinate (h4) at ({55/16},{9*sqrt(3)/16});

\coordinate (h5) at ({188/71},{96*sqrt(3)/71});

\fill[red!25]
    (h1) -- (h2) -- (h4) -- (h5) -- (h3) -- cycle;

\draw[red, thick]
    (h1) -- (h2) -- (h4) -- (h5) -- (h3) -- cycle;

\draw[blue, thick, dashed]
    (e1) -- (e2) -- (e3) -- cycle;

\draw[green!60!black, thick, dotted]
    (2,{2*sqrt(3)/3}) circle ({2*sqrt(3)/3});

\fill[red] (h1) circle (2.5pt);
\fill[red] (h2) circle (2.5pt);
\fill[red] (h3) circle (2.5pt);
\fill[red] (h4) circle (2.5pt);
\fill[red] (h5) circle (2.5pt);

\node[blue, below left]  at (e1) {$e_1=(1,0,0)$};
\node[blue, below right] at (e2) {$e_2=(0,1,0)$};
\node[blue, above]       at (e3) {$e_3=(0,0,1)$};





\node[green!60!black]
    at (2,{2*sqrt(3)/3})
    {$\mathcal{C}\cap\Delta^3$};

\node[blue] at (3,3) {$\Delta^3$};

\node[red] at (2,-0.5)
    {$\operatorname{cone}(H'_d) \cap \Delta^3$};

\end{tikzpicture}
    \caption{Illustration that $H_d'$ does not satisfy the SSC. 
    \label{fig:HsscExample4}}
\end{figure}


In practice, for example in hyperspectral imaging, the issue described above implies that an endmember that is present in a larger proportion in the data will, in most cases, be more difficult\footnote{Note that this issue does not arise if a pure pixel corresponding to that endmember is present in the image.} to identify under the sum-to-one row constraint on $H$. This is the very height of irony: the more abundant a material is in the data, the harder it may be to identify using~\eqref{eq:NMFrowsumtoone}.  


A related issue is that row normalization of $H$ can make $W$ ill-conditioned when endmembers are present in rather different proportions; see the discussions in~\cite{leplat2019blind} and \cite[Section~4.3.3.3, p.~142]{gillis2020}. 
In practice, this makes the MinVol NMF variant with the sum-to-one constraint on the rows of $H$ perform poorly;  see~\cite[Section~4.3.3.5, p.~144]{gillis2020} for a discussion and a numerical example.

}

\section{Optimal $HH^\top$ for MaxVol NMF as $\lambda \rightarrow \infty$} \label{app:optsolA}

Let us introduce the notation $G = HH^\top  \in \mathbb{S}^r$, where $\mathbb{S}^r$ is the set of $r$-by-$r$ symmetric matrices. Now consider 
\begin{mini}
    {G\in \mathbb{S}^r}{f_0(G)=\revise{-}\logdet (G+\delta I)}{\label{eq:maxvol}}{}
    \addConstraint{1_r^\top G 1_r\leq a, G\geq0,}
\end{mini}
where $a>0$ and $\dom f_0=\{G\in\mathbb{S}^r,G\succ-\delta\}$. 
Note that $1_r^\top HH^\top 1_r = n$. 
\revise{Moreover, \eqref{eq:maxvol} is a relaxation since $G$ is not constrained to be positive semidefinite. However, we show in the following that the optimal solution is, which is enough for our purpose.} In fact, we prove that $G=\frac{a}{r}I$ is the unique minimizer of~\eqref{eq:maxvol}. 
To do so, we solve~\eqref{eq:maxvol} through its dual using the conjugate of $f_0$, as in \cite[Section 5.1.6]{boyd2004convex}. 
\begin{definition}
    The conjugate $f^*$ of a function $f:\R^r\rightarrow\R$ is given by
    $$f^*(y)=\sup_{x\in\dom f}\left(y^\top x-f(x)\right).$$
\end{definition}

\noindent Considering the optimization problem with linear inequality and equality constraints
\begin{equation} \label{eq:linconst} 
    \min_x f_0(x) 
    \;  \text{ such that } \; 
    Ax\leq b, \; Cx=d,
\end{equation} 
the conjugate of $f_0$ can be used to write the dual function for~\eqref{eq:linconst} as
\begin{align*}
    g(\lambda,\nu)&= \inf_x\left(f_0(x)+\lambda^\top(Ax-b)+\nu^\top(Cx-d)\right)\\
    &= -b^\top\lambda -d^\top\nu +\inf_x\left(f_0(x)+(A^\top\lambda+C^\top\nu)^\top x\right)\\
    &= -b^\top\lambda-d^\top\nu-f_0^*(-A^\top\lambda-C^\top\nu)\stepcounter{equation}\tag{\theequation}\label{eq:conjindual},
\end{align*}
\revise{where $\lambda$ and $\nu$ are the Lagrange multipliers associated respectively with the inequality constraints and the equality constraints.}
The domain of $g$ follows from the domain of $f_0^*$:
\[
\dom g=\{(\lambda,\nu)|-A^\top\lambda-C^\top\nu\in\dom f_0^*\}.
\]
Let us go back to the conjugate function of $f_0$, defined as
\[
f_0^*(Y)=\sup_{G\succ-\delta}\left(\langle Y,G\rangle+\logdet (G+\delta I)\right).
\] 
We first show that $\langle Y,G\rangle+\logdet (G+\delta I)$ is unbounded above unless $Y\prec0$. If $Y\nprec0$, then $Y$ has an eigenvector $v$, with $\|v\|_2=1$, and eigenvalue $\revise{\mu}\geq0$. Taking $G=I+tvv^\top$ we find that 
\begin{align*}
    \MoveEqLeft[6] \langle Y,G\rangle+\logdet (G+\delta I) \\
    = & \tr Y+t\revise{\mu}+\logdet((1+\delta)I+tvv^\top) \\
    = & \tr Y+t\revise{\mu}+r\log(1+\delta)+\log\left(1+\frac{t}{1+\delta}\right)
\end{align*}
which is unbounded above as $t\rightarrow\infty$.
\noindent Now consider the case $Y\prec0$. We can find the maximizing $G$ by setting the gradient with respect to $G$ equal to zero:
\[\nabla_G(\langle Y,G\rangle+\logdet (G+\delta I))=Y+(G+\delta I)^{-1}=0,\]
which leads to $G=-Y^{-1}-\delta I$. Therefore, we have
\begin{equation}
    f_0^*(Y)=\logdet(-Y^{-1})-r-\delta\tr(Y)
    \label{eq:conjugatef0}
\end{equation}
with $\dom f_0^*=-\mathbb{S}^r_{++}$. 

Applying the result in~\eqref{eq:conjindual}, the dual function for problem~\eqref{eq:maxvol}\revise{, which has a single scalar inequality $1_r^\top G 1_r\leq a$ with multiplier $\lambda_1\in\R_+$, an entrywise inequality $G\geq0$ with multiplier matrix $\lambda_2\in\R_+^{r\times r}$, and no equality constraints, is given by 
$$g(\lambda_1,\lambda_2)=\logdet\bigl(\lambda_1 J-\lambda_2\bigr) +r+\delta\tr(\lambda_1 J -\lambda_2)-\lambda_1 a$$ if $\lambda_1 J-\lambda_2\succ0$, where $J$ is the matrix of all ones of appropriate dimension, and $g(\lambda_1,\lambda_2)=-\infty$ otherwise.} Let 
$$\lambda_1^*=\frac{r}{a}\left(1+\frac{\delta}{\frac{a}{r}+\delta}\right),~\lambda_2^*=\frac{r}{a}\left(1+\frac{\delta}{\frac{a}{r}+\delta}\right)J-\frac{1}{\frac{a}{r}+\delta}I$$
and $G^*=\frac{a}{r}I$. We have $f_0(G^*)=g(\lambda_1^*,\lambda_2^*)=-r\log(\frac{a}{r}+\delta)$, meaning that there is no duality gap and that $G^*$ is a solution of~\eqref{eq:maxvol}. Finally, $G^*$ is the unique solution because $f_0$ is strictly convex.


\section{Optimal {$\widetilde{H}\widetilde{H}^\top$} for MaxVol NMF as $\lambda \rightarrow \infty$} 
\label{app:optsolB}

Let us introduce the variable $G = \widetilde{H}\widetilde{H}^\top$, and 
show that the problem
\begin{mini}
    {G \in \mathbb{S}^r}{f_0(G)=\revise{-}\logdet (G+\delta I)}{\label{eq:normmaxvol}}{}
    \addConstraint{\diag(G)=1_r, 0\leq G \leq 1,} 
\end{mini}
where $\dom f_0=\mathbb{S}^r_{++}$, has $G=I$ as a unique minimizer. Again, we solve this problem through its dual using the conjugate of $f_0$, which has already been computed in \eqref{eq:conjugatef0}. First, \eqref{eq:normmaxvol} can be reformulated as 
\begin{mini}
    {G\in \mathbb{S}^r}{f_0(G)=\revise{-}\logdet (G+\delta I)}{\label{eq:normmaxvol_linearcons}}{}
    \addConstraint{\langle E_{ii},G\rangle=1 \text{ for all } i}
    \addConstraint{\langle -E_{ij},G\rangle\leq0 \text{ for all } i,j}
    \addConstraint{\langle E_{ij},G\rangle\leq1 \text{ for all } i,j.}
\end{mini}
Using again~\eqref{eq:conjindual}, we can write the dual of~\eqref{eq:normmaxvol_linearcons} with the conjugate of $f_0$: 
\begin{multline*}
    g(\lambda_1,\lambda_2,\nu)=\logdet\bigl(\diag(\nu)+\lambda_2-\lambda_1\bigr)-\langle J,\lambda_2 \rangle - 1_r^\top\nu\\
    +r-\delta\tr\bigl(\diag(\nu)+\lambda_2-\lambda_1\bigr)
\end{multline*}
if $\diag(\nu)+\lambda_2-\lambda_1 \succ0$, and $g(\lambda_1,\lambda_2,\nu)=-\infty$ otherwise,
where $\lambda_1\in\R^{r \times r}_+$ is the multiplier of the inequality constraints $\langle -E_{ij},G\rangle\leq0$ for all $(i,j)$, $\lambda_2\in\R^{r \times r}_+$ is the multiplier of the inequality constraints $\langle E_{ij},G\rangle\leq1$ for all $(i,j)$, and $\nu\in\R^r$ is the multiplier of the equality constraints $\langle E_{ii},G\rangle=1$ for all $i$. Let $\lambda_1^*=0,\lambda_2^*=0,\nu^*=\frac{1}{1+\delta}1_r$ and $G^*=I$. We have $f_0(G^*)=g(\lambda_1^*,\lambda_2^*,\nu^*)=-r\log(1+\delta)$, meaning that there is no duality gap and that $G^*$ is a solution of \eqref{eq:normmaxvol}. Finally, $G^*$ is the unique solution because $f_0$ is strictly convex. 

Let us prove that the minimum of the volume criterion is reached when $\widetilde{H}\widetilde{H}^\top = J$ 
by solving 
\begin{mini}
    {G\in \mathbb{S}^r}{f_1(G)=\logdet (G+\delta I)}{\label{eq:minnormmaxvol_linearcons}}{}
    \addConstraint{\langle E_{ii},G\rangle=1 \text{ for all } i}
    \addConstraint{\langle -E_{ij},G\rangle\leq0 \text{ for all } i,j}
    \addConstraint{\langle E_{ij},G\rangle\leq1 \text{ for all } i,j.}
\end{mini} 
Skipping the details, the conjugate of $f_1$ is $$f_1^*(Y)=r-\delta\tr(Y)+\logdet(Y)$$
with $\dom f_1^*=\mathbb{S}^r_{++}$. From \eqref{eq:conjindual}, the dual $g$ of \eqref{eq:minnormmaxvol_linearcons} is 
\begin{multline*}
    g(\lambda_1,\lambda_2,\nu)=\delta\tr(\lambda_1-\lambda_2-\diag(\nu))-1_r^\top\nu-r\\-\logdet(\lambda_1-\lambda_2-\diag(\nu))-\langle J,\lambda_2\rangle
\end{multline*}
if $\lambda_1-\lambda_2-\diag(\nu) \succ0$, and $g(\lambda_1,\lambda_2,\nu)=-\infty$ otherwise,
where $\lambda_1\in\R^{r \times r}_+$, $\lambda_2\in\R^{r \times r}_+$ and $\nu\in\R^r$. Let $\lambda_1^*=0,\lambda_2^*=\frac{1}{\delta(r+\delta)}J,\nu^*=-\frac{1}{\delta}1_r$ and $G^*=J$. 
We have $f_1(G^*)=g(\lambda_1^*,\lambda_2^*,\nu^*)=\log((r+\delta)\delta^{r-1})$, meaning that there is no duality gap and that $G^*$ is a solution of \eqref{eq:normmaxvol}. Finally, $G^*$ is the unique solution because it is a vertex of the polyhedral set defined by the constraints in \eqref{eq:minnormmaxvol_linearcons}, and that the minimum of a strictly concave function like $f_1$ in a polyhedral set is reached on one the vertex of this set.

\bibliographystyle{abbrv}
\bibliography{thesis}

\newpage

\section*{Supplementary Material}

In this Supplementary Material, we report the results  for MinVol NMF and  MaxVol NMF on data sets which are not included in the paper. The parameters for these two models, namely $\lambda$ and $\delta$, were chosen to achieve good results. 
Table~\ref{tab:figures} summarizes the content of these figures. 

\vfill

\begin{table}[H]
    \centering
    \begin{tabular}{r|cccc}
                     &  Moffet   & Samson    & Urban   &  Jasper \\\hline
         MinVol NMF  &   Fig.~\ref{fig:MinVol_moffett}   &    Fig.~\ref{fig:MinVol_samson}        &  Fig.~\ref{fig:MinVol_urban}      & Fig.~\ref{fig:MinVol_jasper}      \\ 
         MaxVol NMF  &    Fig.~\ref{fig:MaxVol_moffett}       &     Fig.~\ref{fig:MaxVol_samson}      &   Fig.~\ref{fig:MaxVol_urban}      & Fig.~\ref{fig:MaxVol_jasper}
    \end{tabular}
    \caption{Content of this Supplementary Material: results from MinVol and MaxVol NMF on four data sets. 
    \label{tab:figures}
    }
\end{table}

\vfill

\begin{figure}[htbp!]
  \centering

  \fbox{\includegraphics[width=0.9\linewidth]{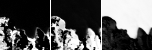}}

  \begin{tikzpicture}
    \begin{groupplot}[
        group style={
          group name=my plots,
          group size= 3 by 1,
          x descriptions at=edge bottom,
          y descriptions at=edge left,
          horizontal sep=0.0cm,
          vertical sep=0.cm,
        },
        scale only axis,height=2cm,width=0.9\linewidth/3,no markers,grid=major,xlabel={},ylabel={},ymin=0,ymax=1,
        ytick={0,1},
    ]
      \nextgroupplot
      \addplot [line width=1pt]
        table [x expr=\coordindex+1, y index=0]
        {maxvolnmf/figures/MinVol_moffett.txt};

      \nextgroupplot
      \addplot [line width=1pt,red]
        table [x expr=\coordindex+1, y index=1]
        {maxvolnmf/figures/MinVol_moffett.txt};

      \nextgroupplot
      \addplot [line width=1pt,blue]
        table [x expr=\coordindex+1, y index=2]
        {maxvolnmf/figures/MinVol_moffett.txt};
    \end{groupplot}
  \end{tikzpicture}

  \caption{Abundance maps and endmembers (tree, {\color{red}soil} and {\color{blue}water}) by MinVol NMF on Moffett, with $\lambda=1$ and $\delta=0.5$.}
  \label{fig:MinVol_moffett}
\end{figure}

\vfill

\vspace{-0.5cm}

\vfill

\begin{figure}[htbp!]
\centering
\fbox{\includegraphics[width=0.9\linewidth]{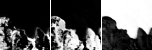}}

\begin{tikzpicture}
    \begin{groupplot}[
        group style={
                group name=my plots,
                group size= 3 by 1,
                x descriptions at=edge bottom,
                y descriptions at=edge left,
                horizontal sep=0.0cm,
                vertical sep=0.cm,
            },
            scale only axis,height=2cm,width=0.9\linewidth/3,no markers,grid=major,xlabel={},ylabel={},ymin=0,ymax=1,
            ytick={0,1},
        ]
        \nextgroupplot[]
        \addplot [line width = 1pt] table [x expr=\coordindex+1, y index=0] {maxvolnmf/figures/MaxVol_moffett.txt};
        \nextgroupplot[]
        \addplot [line width = 1pt,red] table [x expr=\coordindex+1, y index=1] {maxvolnmf/figures/MaxVol_moffett.txt};
        \nextgroupplot[]
        \addplot [line width = 1pt,blue] table [x expr=\coordindex+1, y index=2] {maxvolnmf/figures/MaxVol_moffett.txt};
    \end{groupplot}
\end{tikzpicture}
  
  \caption[Abundance maps and endmembers by MaxVol NMF on Moffett]{Abundance maps and endmembers (tree, {\color{red}soil} and {\color{blue}water}) by MaxVol NMF on Moffett, with $\lambda=1$ and $\delta=1$.}
  \label{fig:MaxVol_moffett}
\end{figure}

\begin{figure}[htbp!]
 \centering\fbox{\includegraphics[width=0.9\linewidth]{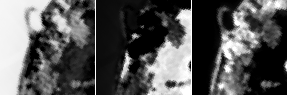}}

  \begin{tikzpicture}
    \begin{groupplot}[
        group style={
          group name=my plots,
          group size= 3 by 1,
          x descriptions at=edge bottom,
          y descriptions at=edge left,
          horizontal sep=0.0cm,
          vertical sep=0.cm,
        },
        scale only axis,width=0.9\linewidth/3,no markers,grid=major,xlabel={},ylabel={},ymin=0,ymax=1,ytick={0,1},
      ]
      \nextgroupplot
      \addplot [line width = 1pt,blue] table [x expr=\coordindex+1, y index=0] {maxvolnmf/figures/MinVol_samson.txt};
      \nextgroupplot
      \addplot [line width = 1pt] table [x expr=\coordindex+1, y index=1] {maxvolnmf/figures/MinVol_samson.txt};
      \nextgroupplot
      \addplot [line width = 1pt,red] table [x expr=\coordindex+1, y index=2] {maxvolnmf/figures/MinVol_samson.txt};
    \end{groupplot}
  \end{tikzpicture}

  \caption[Abundance maps and endmembers by MinVol NMF on Samson]{Abundance maps and endmembers ({\color{blue}water}, soil and {\color{red}tree}) by MinVol NMF on Samson, with $\lambda=1$ and $\delta=0.5$.}
  \label{fig:MinVol_samson}
\end{figure}

\begin{figure}[htbp!]
\centering\fbox{\includegraphics[width=0.9\linewidth]{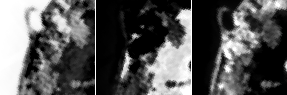}}

  \begin{tikzpicture}
    \begin{groupplot}[
        group style={
          group name=my plots,
          group size= 3 by 1,
          x descriptions at=edge bottom,
          y descriptions at=edge left,
          horizontal sep=0.0cm,
          vertical sep=0.cm,
        },
        scale only axis,width=0.9\linewidth/3,no markers,grid=major,xlabel={},ylabel={},ymin=0,ymax=1,ytick={0,1},
      ]
      \nextgroupplot
      \addplot [line width = 1pt,blue] table [x expr=\coordindex+1, y index=0] {maxvolnmf/figures/MaxVol_samson.txt};
      \nextgroupplot
      \addplot [line width = 1pt] table [x expr=\coordindex+1, y index=1] {maxvolnmf/figures/MaxVol_samson.txt};
      \nextgroupplot
      \addplot [line width = 1pt,red] table [x expr=\coordindex+1, y index=2] {maxvolnmf/figures/MaxVol_samson.txt};
    \end{groupplot}
  \end{tikzpicture}
  
  \caption[Abundance maps and endmembers by N-MaxVol NMF on Samson]{Abundance maps and endmembers ({\color{blue}water}, soil and {\color{red}tree}) by MaxVol NMF on Samson, with $\lambda=1$ and $\delta=1$.}
  \label{fig:MaxVol_samson}
\end{figure}

\begin{figure*}[htbp!]
  \begin{subfigure}{\linewidth}
    \caption{$r=4, \lambda=0.1, \delta=0.1$}
    \centering\fbox{\includegraphics[width=0.9\linewidth]{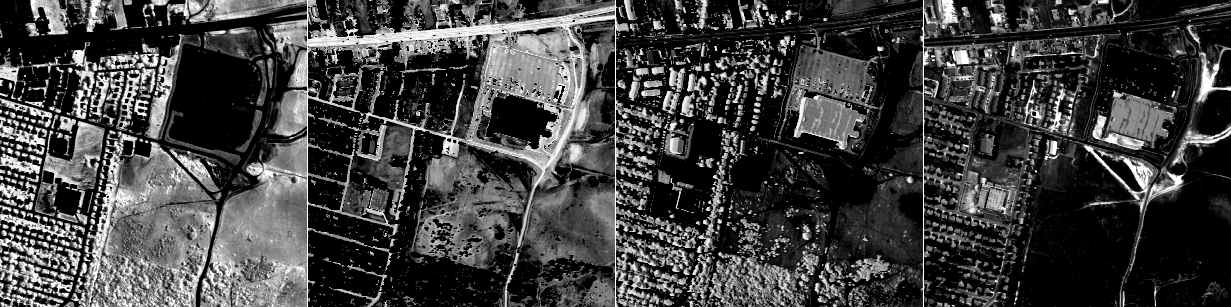}}

    \begin{tikzpicture}
      \begin{groupplot}[
          group style={
            group name=my plots,
            group size= 4 by 1,
            x descriptions at=edge bottom,
            y descriptions at=edge left,
            horizontal sep=0.0cm,
            vertical sep=0.cm,
          },
          scale only axis,width=0.9\linewidth/4,no markers,grid=major,xlabel={},ylabel={},ymin=0,ymax=1,ytick={0,1},
        ]
        \nextgroupplot
        \addplot [line width = 1pt] table [x expr=\coordindex+1, y index=0] {maxvolnmf/figures/MinVol_urban4.txt};
        \nextgroupplot
        \addplot [line width = 1pt,teal] table [x expr=\coordindex+1, y index=1] {maxvolnmf/figures/MinVol_urban4.txt};
        \nextgroupplot
        \addplot [line width = 1pt,red] table [x expr=\coordindex+1, y index=2] {maxvolnmf/figures/MinVol_urban4.txt};
        \nextgroupplot
        \addplot [line width = 1pt,blue] table [x expr=\coordindex+1, y index=3] {maxvolnmf/figures/MinVol_urban4.txt};
      \end{groupplot}
    \end{tikzpicture}
  \end{subfigure}

  \begin{subfigure}{\linewidth}
    \caption{$r=5, \lambda=0.1, \delta=0.1$}
    \centering\fbox{\includegraphics[width=0.9\linewidth]{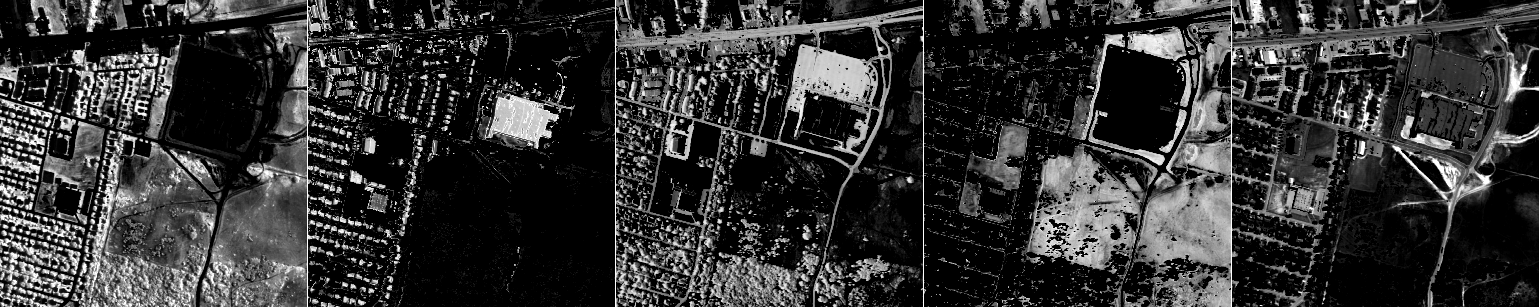}}

    \begin{tikzpicture}
      \begin{groupplot}[
          group style={
            group name=my plots,
            group size= 5 by 1,
            x descriptions at=edge bottom,
            y descriptions at=edge left,
            horizontal sep=0.0cm,
            vertical sep=0.cm,
          },
          scale only axis,width=0.9\linewidth/5,no markers,grid=major,xlabel={},ylabel={},ymin=0,ymax=1,ytick={0,1},
        ]
        \nextgroupplot
        \addplot [line width = 1pt,teal] table [x expr=\coordindex+1, y index=0] {maxvolnmf/figures/MinVol_urban5.txt};
        \nextgroupplot
        \addplot [line width = 1pt,red] table [x expr=\coordindex+1, y index=1] {maxvolnmf/figures/MinVol_urban5.txt};
        \nextgroupplot
        \addplot [line width = 1pt,orange] table [x expr=\coordindex+1, y index=2] {maxvolnmf/figures/MinVol_urban5.txt};
        \nextgroupplot
        \addplot [line width = 1pt] table [x expr=\coordindex+1, y index=3] {maxvolnmf/figures/MinVol_urban5.txt};
        \nextgroupplot
        \addplot [line width = 1pt,blue] table [x expr=\coordindex+1, y index=4] {maxvolnmf/figures/MinVol_urban5.txt};
      \end{groupplot}
    \end{tikzpicture}
  \end{subfigure}

  \begin{subfigure}{\linewidth}
    \caption{$r=6, \lambda=0.1, \delta=0.1$}
    \centering\fbox{\includegraphics[width=0.9\linewidth]{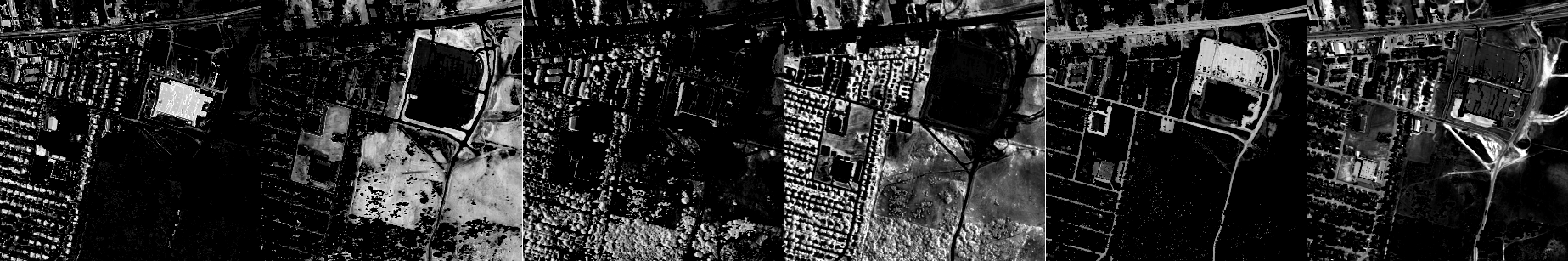}}

    \begin{tikzpicture}
      \begin{groupplot}[
          group style={
            group name=my plots,
            group size= 6 by 1,
            x descriptions at=edge bottom,
            y descriptions at=edge left,
            horizontal sep=0.0cm,
            vertical sep=0.cm,
          },
          scale only axis,width=0.9\linewidth/6,no markers,grid=major,xlabel={},ylabel={},ymin=0,ymax=1,ytick={0,1},
        ]
        \nextgroupplot
        \addplot [line width = 1pt] table [x expr=\coordindex+1, y index=0] {maxvolnmf/figures/MinVol_urban6.txt};
        \nextgroupplot
        \addplot [line width = 1pt,blue] table [x expr=\coordindex+1, y index=1] {maxvolnmf/figures/MinVol_urban6.txt};
        \nextgroupplot
        \addplot [line width = 1pt,olive] table [x expr=\coordindex+1, y index=2] {maxvolnmf/figures/MinVol_urban6.txt};
        \nextgroupplot
        \addplot [line width = 1pt,red] table [x expr=\coordindex+1, y index=3] {maxvolnmf/figures/MinVol_urban6.txt};
        \nextgroupplot
        \addplot [line width = 1pt,orange] table [x expr=\coordindex+1, y index=4] {maxvolnmf/figures/MinVol_urban6.txt};
        \nextgroupplot
        \addplot [line width = 1pt,teal] table [x expr=\coordindex+1, y index=5] {maxvolnmf/figures/MinVol_urban6.txt};
      \end{groupplot}
    \end{tikzpicture}
  \end{subfigure}
  \caption[Abundance maps and endmembers by MinVol NMF on Urban depending on $r$]{Abundance maps and endmembers
  by MinVol NMF on Urban, depending on $r$, $\lambda$ and $\delta$.}
  \label{fig:MinVol_urban}
\end{figure*}

\begin{figure*}[htbp!]
  \begin{subfigure}{\linewidth}
    \caption{$r=4, \lambda=0.2, \delta=2$}
    \centering\fbox{\includegraphics[width=0.9\linewidth]{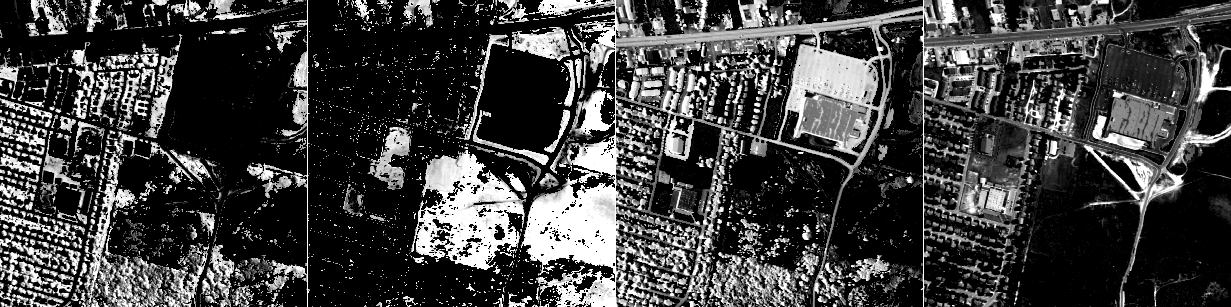}}

    \begin{tikzpicture}
      \begin{groupplot}[
          group style={
            group name=my plots,
            group size= 4 by 1,
            x descriptions at=edge bottom,
            y descriptions at=edge left,
            horizontal sep=0.0cm,
            vertical sep=0.cm,
          },
          scale only axis,width=0.9\linewidth/4,no markers,grid=major,xlabel={},ylabel={},ymin=0,ymax=1,ytick={0,1},
        ]
        \nextgroupplot
        \addplot [line width = 1pt] table [x expr=\coordindex+1, y index=0] {maxvolnmf/figures/MaxVol_urban4.txt};
        \nextgroupplot
        \addplot [line width = 1pt,teal] table [x expr=\coordindex+1, y index=1] {maxvolnmf/figures/MaxVol_urban4.txt};
        \nextgroupplot
        \addplot [line width = 1pt,red] table [x expr=\coordindex+1, y index=2] {maxvolnmf/figures/MaxVol_urban4.txt};
        \nextgroupplot
        \addplot [line width = 1pt,blue] table [x expr=\coordindex+1, y index=3] {maxvolnmf/figures/MaxVol_urban4.txt};
      \end{groupplot}
    \end{tikzpicture}
  \end{subfigure}

  \begin{subfigure}{\linewidth}
    \caption{$r=5, \lambda=0.5, \delta=2$}
    \centering\fbox{\includegraphics[width=0.9\linewidth]{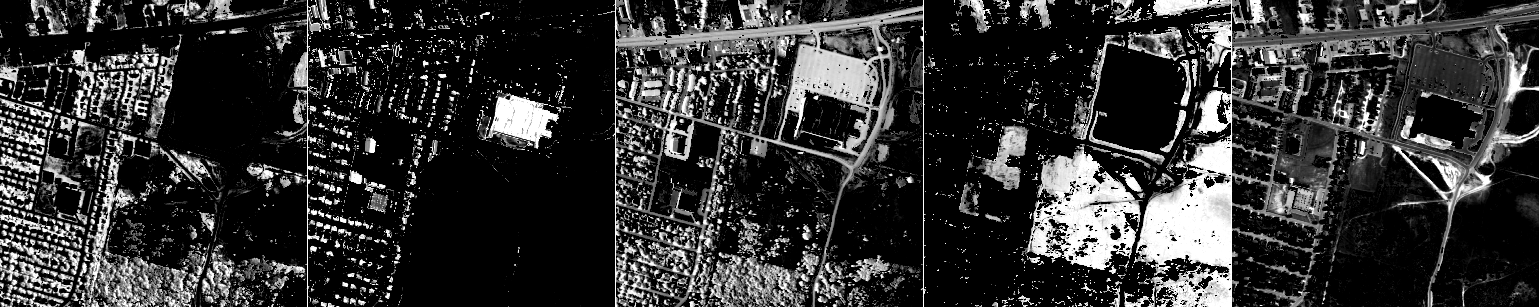}}

    \begin{tikzpicture}
      \begin{groupplot}[
          group style={
            group name=my plots,
            group size= 5 by 1,
            x descriptions at=edge bottom,
            y descriptions at=edge left,
            horizontal sep=0.0cm,
            vertical sep=0.cm,
          },
          scale only axis,width=0.9\linewidth/5,no markers,grid=major,xlabel={},ylabel={},ymin=0,ymax=1,ytick={0,1},
        ]
        \nextgroupplot
        \addplot [line width = 1pt,teal] table [x expr=\coordindex+1, y index=0] {maxvolnmf/figures/MaxVol_urban5.txt};
        \nextgroupplot
        \addplot [line width = 1pt,red] table [x expr=\coordindex+1, y index=1] {maxvolnmf/figures/MaxVol_urban5.txt};
        \nextgroupplot
        \addplot [line width = 1pt,orange] table [x expr=\coordindex+1, y index=2] {maxvolnmf/figures/MaxVol_urban5.txt};
        \nextgroupplot
        \addplot [line width = 1pt] table [x expr=\coordindex+1, y index=3] {maxvolnmf/figures/MaxVol_urban5.txt};
        \nextgroupplot
        \addplot [line width = 1pt,blue] table [x expr=\coordindex+1, y index=4] {maxvolnmf/figures/MaxVol_urban5.txt};
      \end{groupplot}
    \end{tikzpicture}
  \end{subfigure}

  \begin{subfigure}{\linewidth}
    \caption{$r=6, \lambda=0.5, \delta=2$}
    \centering\fbox{\includegraphics[width=0.9\linewidth]{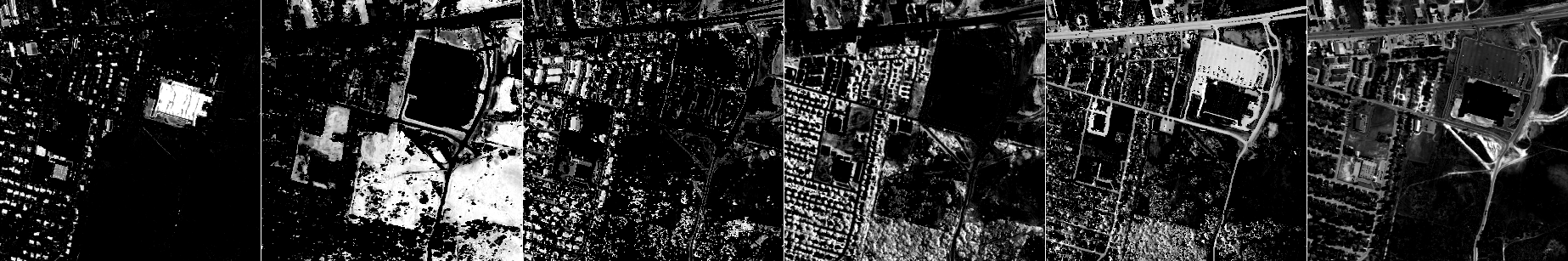}}

    \begin{tikzpicture}
      \begin{groupplot}[
          group style={
            group name=my plots,
            group size= 6 by 1,
            x descriptions at=edge bottom,
            y descriptions at=edge left,
            horizontal sep=0.0cm,
            vertical sep=0.cm,
          },
          scale only axis,width=0.9\linewidth/6,no markers,grid=major,xlabel={},ylabel={},ymin=0,ymax=1,ytick={0,1},
        ]
        \nextgroupplot
        \addplot [line width = 1pt] table [x expr=\coordindex+1, y index=0] {maxvolnmf/figures/MaxVol_urban6.txt};
        \nextgroupplot
        \addplot [line width = 1pt,blue] table [x expr=\coordindex+1, y index=1] {maxvolnmf/figures/MaxVol_urban6.txt};
        \nextgroupplot
        \addplot [line width = 1pt,olive] table [x expr=\coordindex+1, y index=2] {maxvolnmf/figures/MaxVol_urban6.txt};
        \nextgroupplot
        \addplot [line width = 1pt,red] table [x expr=\coordindex+1, y index=3] {maxvolnmf/figures/MaxVol_urban6.txt};
        \nextgroupplot
        \addplot [line width = 1pt,orange] table [x expr=\coordindex+1, y index=4] {maxvolnmf/figures/MaxVol_urban6.txt};
        \nextgroupplot
        \addplot [line width = 1pt,teal] table [x expr=\coordindex+1, y index=5] {maxvolnmf/figures/MaxVol_urban6.txt};
      \end{groupplot}
    \end{tikzpicture}
  \end{subfigure}
  \caption[Abundance maps and endmembers by MaxVol NMF on Urban depending on $r$]{Abundance maps and endmembers
  by MaxVol NMF on Urban, depending on $r$, $\lambda$ and $\delta$.}
  \label{fig:MaxVol_urban}
\end{figure*}

\begin{figure*}[htbp!]
  \centering\fbox{\includegraphics[width=0.9\linewidth]{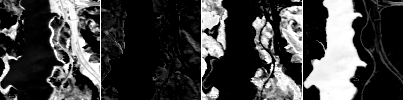}}

  \begin{tikzpicture}
    \begin{groupplot}[
        group style={
          group name=my plots,
          group size= 4 by 1,
          x descriptions at=edge bottom,
          y descriptions at=edge left,
          horizontal sep=0.0cm,
          vertical sep=0.cm,
        },
        scale only axis,width=0.9\linewidth/4,no markers,grid=major,xlabel={},ylabel={},ymin=0,ymax=1,ytick={0,1},
      ]
      \nextgroupplot
      \addplot [line width = 1pt] table [x expr=\coordindex+1, y index=0] {maxvolnmf/figures/MinVol_jasper.txt};
      \nextgroupplot
      \addplot [line width = 1pt,blue] table [x expr=\coordindex+1, y index=1] {maxvolnmf/figures/MinVol_jasper.txt};
      \nextgroupplot
      \addplot [line width = 1pt,red] table [x expr=\coordindex+1, y index=2] {maxvolnmf/figures/MinVol_jasper.txt};
      \nextgroupplot
      \addplot [line width = 1pt,teal] table [x expr=\coordindex+1, y index=3] {maxvolnmf/figures/MinVol_jasper.txt};
    \end{groupplot}
  \end{tikzpicture}
  \caption[Abundance maps and endmembers by MinVol NMF with $r=4$ on Jasper]{Abundance maps and endmembers (road and soil, {\color{red}tree and soil}, {\color{teal}water}) by MinVol NMF with $r=4$ on Jasper, with $\lambda=0.5$ and $\delta=1$.}
  \label{fig:MinVol_jasper}
\end{figure*}

\begin{figure*}[htbp!]
  \centering\fbox{\includegraphics[width=0.9\linewidth]{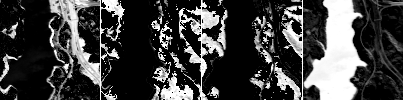}}

  \begin{tikzpicture}
    \begin{groupplot}[
        group style={
          group name=my plots,
          group size= 4 by 1,
          x descriptions at=edge bottom,
          y descriptions at=edge left,
          horizontal sep=0.0cm,
          vertical sep=0.cm,
        },
        scale only axis,width=0.9\linewidth/4,no markers,grid=major,xlabel={},ylabel={},ymin=0,ymax=1,ytick={0,1},
      ]
      \nextgroupplot
      \addplot [line width = 1pt] table [x expr=\coordindex+1, y index=0] {maxvolnmf/figures/MaxVol_jasper.txt};
      \nextgroupplot
      \addplot [line width = 1pt,blue] table [x expr=\coordindex+1, y index=1] {maxvolnmf/figures/MaxVol_jasper.txt};
      \nextgroupplot
      \addplot [line width = 1pt,red] table [x expr=\coordindex+1, y index=2] {maxvolnmf/figures/MaxVol_jasper.txt};
      \nextgroupplot
      \addplot [line width = 1pt,teal] table [x expr=\coordindex+1, y index=3] {maxvolnmf/figures/MaxVol_jasper.txt};
    \end{groupplot}
  \end{tikzpicture}
  \caption[Abundance maps and endmembers by MaxVol NMF with $r=4$ on Jasper]{Abundance maps and endmembers (road, {\color{blue}tree and soil}, {\color{red}tree}, {\color{teal}water}) by MaxVol NMF with $r=4$ on Jasper, with $\lambda=0.5$ and $\delta=1$.}
  \label{fig:MaxVol_jasper}
\end{figure*}

\end{document}